\documentclass[twoside]{article}

\usepackage[accepted]{aistats2024}

\usepackage[square]{natbib}

\usepackage[utf8]{inputenc}             
\usepackage[T1]{fontenc}                
\usepackage[pagebackref]{hyperref}      
\usepackage{url}                        
\usepackage{booktabs}                   
\usepackage{amsfonts}                   
\usepackage{nicefrac}                   
\usepackage{microtype}                  
\usepackage{xcolor}                     
\usepackage{enumitem}

\usepackage{amsmath}
\usepackage{amsthm}
\usepackage{array}
\usepackage{bbm}
\usepackage{caption}
\usepackage[capitalise,noabbrev]{cleveref}
\usepackage{dsfont}
\usepackage{graphicx}
\usepackage{makecell}
\usepackage{multirow}
\usepackage{subcaption}
\usepackage{tabularx}

\newcommand{\pdf}[2]{p_{#1} \left( #2 \right)}
\newcommand{\cdf}[2]{P_{#1} \left( #2 \right)}
\newcommand{\tildecdf}[2]{\tilde{P}_{#1} \left( #2 \right)}
\newcommand{\empcdf}[2]{P^{\text{emp}}_{#1} \left( #2 \right)}
\newcommand{\tildeempcdf}[2]{\tilde{P}^{\text{emp}}_{#1} \left( #2 \right)}
\newcommand{\pr}[1]{\mathds{P} \left\{ #1 \right\}}
\newcommand{\ev}[1]{\mathds{E} \left[ #1 \right]}
\newcommand{\ind}[1]{\mathds{1} \left\{ #1 \right\}}
\newcommand{\size}[1]{\left\vert #1 \right\vert}
\newcommand{\bin}[2]{B \left( #1 , #2 \right)}
\newcommand{\ps}{\hat{C}^{R}_{\alpha} ( X_{n + 1} ; Z_{1 : n} )}
\newcommand{\thr}{\tau_{\alpha} ( R_{1 : n} )}
\newcommand{\mf}{\#_{R} ( r )}
\newcommand{\epss}{\ev{\size{\ps}}}

\newtheorem{theorem}{Theorem}

\newtheorem{corollary}[theorem]{Corollary}
\theoremstyle{definition}

\theoremstyle{remark}

\newcolumntype{C}[1]{>{\centering\arraybackslash}p{#1}}

\begin{document}


\twocolumn[

\aistatstitle{On the Expected Size of Conformal Prediction Sets}

\aistatsauthor{Guneet S. Dhillon \And George Deligiannidis \And Tom Rainforth}

\aistatsaddress{Department of Statistics \\ University of Oxford \\ \texttt{guneet.dhillon@stats.ox.ac.uk} \And Department of Statistics \\ University of Oxford \\ \texttt{deligian@stats.ox.ac.uk} \And Department of Statistics \\ University of Oxford \\ \texttt{rainforth@stats.ox.ac.uk}}

]

\begin{abstract}
While conformal predictors reap the benefits of rigorous statistical guarantees on their error frequency, the size of their corresponding prediction sets is critical to their practical utility. Unfortunately, there is currently a lack of finite-sample analysis and guarantees for their prediction set sizes. To address this shortfall, we theoretically quantify the expected size of the prediction sets under the split conformal prediction framework. As this precise formulation cannot usually be calculated directly, we further derive point estimates and high-probability interval bounds that can be empirically computed, providing a practical method for characterizing the expected set size. We corroborate the efficacy of our results with experiments on real-world datasets for both regression and classification problems.
\end{abstract}

\section{INTRODUCTION}
\label{section:intro}

Imagine a company that recognizes a market where multiple businesses are interested in the same task, e.g.,~the inspection and quality control of manufactured goods.
Seeing this opportunity, the company decides to provide an AI-powered solution, to help automate and streamline the process for these customers.

A major characteristic that customers would want to know is:~\emph{how often will the system make errors?}
Moreover, they will often want to constrain the frequency of errors to a level acceptable for their particular use case.
The conformal prediction framework~\citep{vovk2005algorithmic,shafer2008tutorial} fulfills this requirement.
Instead of a single label, it predicts a set of labels (based on past experiences) and rigorously guarantees that the error of the predicted set not containing the true label is bound to a user-specified level.
Notably, the split conformal prediction framework~\citep{papadopoulos2002inductive,vovk2005algorithmic,lei2015conformal} further provides computational efficiency for practical deployment.

A natural follow-up question that customers might ask 
is:~\emph{how big are the prediction sets expected to be?}
This is a valid concern.
For instance, a naive prediction of the entire label space achieves zero error; however, such predictions are not useful.
Thus, the sizes of the prediction sets are of significant relevance in practice.

Existing works have considered the asymptotic behavior of the expected size of prediction sets, analyzing conformal predictors in the context of statistical optimality~\citep{lei2013distribution-sets,lei2013distribution-regression,vovk2014conformal,lei2014classification,lei2015conformal,vovk2016criteria,sadinle2019least}.
In practice, however, we are concerned with prediction set sizes in the finite-sample setting.
In fact, in most practical applications of conformal prediction---such as image classification~\citep{angelopoulos2021uncertainty,fisch2021few}, natural language processing~\citep{fisch2021efficient,fisch2021few,schuster2021consistent}, drug discovery~\citep{fisch2021efficient,fisch2021few}, clinical trials~\citep{lu2021distribution,lu2022fair}, robotics~\citep{dixit2023adaptive,lindemann2023safe}, and election polling~\citep{cherian2020washington}---algorithms are compared based on:~(i) their frequencies of error, and (ii) the expected sizes of their prediction sets.

Currently, this expected size is empirically estimated by averaging the sizes of the constructed prediction sets over multiple runs, via Monte Carlo averaging.
For our hypothetical company and customers, this amounts to the company collecting labeled data from multiple customers and running predictions repeatedly; this is an expensive procedure.
Additionally, each customer will have a different set of parameters---e.g.,~different frequency of error requirements, different number of labeled data, etc.---each resulting in a different expected set size.
To provide a satisfactory answer to each customer, the company will have to run this procedure separately for each set of values.
As such, it is often not feasible to provide customers with an indication of the expected prediction set size using this approach.

To overcome these shortcomings, we theoretically quantify the expected size of split conformal prediction sets.
As computing this quantity directly is often intractable in practice, we derive procedures to empirically approximate it.
Our proposed procedures require data to be collected~\emph{only once} to provide a point estimate and high-probability interval bounds for the expected prediction set size.
Consequently, our hypothetical company no longer requires access to labeled data from multiple customers and repeated runs of the conformal algorithm.
Instead, the company could use a single set of pre-collected in-house data to compute both point and interval estimates for the expected size of the prediction sets constructed by its proposed (conformal) system.
From the customer's perspective, this information allows them to reliably evaluate the company's system and further determine whether to use it or not.

In summary, our contributions are as follows:
\begin{itemize}
    \item
    We theoretically quantify the expected size of the prediction sets constructed under the split conformal prediction framework (cf.~\cref{section:size-quantification}).
    \item
    We derive practical point estimates and high-probability intervals for the above (cf.~\cref{section:size-estimation}).
    \item
    We illustrate the efficacy of our results experimentally on regression and classification (cf.~\cref{section:experiments}).
\end{itemize}

\section{BACKGROUND}
\label{section:background}

We are concerned with supervised learning, where we are provided with labeled data and want to predict the label for new test inputs.
The split conformal prediction framework~\citep{papadopoulos2002inductive,vovk2005algorithmic,lei2015conformal} predicts a set of labels for a given test input such that the probability of error, i.e.,~the predicted set not containing the true label, is guaranteed to be bound at a user-specified level.
This is achieved by first splitting the labeled data into training and calibration data.
The training data trains a (non-conformity) scoring function, which is used to compute scores on the calibration data; then a threshold is determined using these calibration scores.
The algorithm uses the computed score threshold to construct the prediction set for a new test input.

Formally, we denote $Z_{i} = ( X_{i} , Y_{i} ) \in \mathcal{X} \times \mathcal{Y}$, for $i = 1 , \ldots , n + 1$, to be $n + 1$ data points sampled i.i.d.~from an arbitrary distribution over feature and label spaces $\mathcal{X}$ and $\mathcal{Y}$ respectively.
We treat $Z_{1 : n} = \{ Z_{1} , \ldots , Z_{n} \}$ as the calibration data and $Z_{n + 1}$ as the test datum; the training data is used to implement the scoring function as discussed later.
For a significance level $\alpha \in ( 0 , 1 )$, we want to predict a set of labels $\hat{C}_{\alpha} ( X_{n + 1} ; Z_{1 : n} ) \subseteq \mathcal{Y}$ such that the probability of error is bound by $\alpha$, i.e.,
\begin{equation}
    \pr{Y_{n + 1} \not\in \hat{C}_{\alpha} \left( X_{n + 1} ; Z_{1 : n} \right)} \leq \alpha.
\label{equation:conformal-error}
\end{equation}
This is a marginal probability, taken over both the test datum $Z_{n + 1}$ and the calibration data $Z_{1 : n}$.

The split conformal framework achieves this by reasoning about a non-conformity function~\citep{vovk2005algorithmic,shafer2008tutorial}.
We denote this function by $R : \mathcal{X} \times \mathcal{Y} \mapsto \mathcal{R}$, which maps a data point to a real-valued non-conformity score ($\mathcal{R} \subseteq \mathds{R}$).
This function quantifies how non-conforming a data point is.
If $R ( x , y )$ is small, then $( x , y )$ is conforming; conversely, if $R ( x , y )$ is large, then $( x , y )$ is non-conforming or atypical.
This function is implemented as a distance function $L$ between the label $y$ and its prediction $\hat{y} = M ( x )$ obtained from a machine learning model $M$.
The model $M$ is trained on training data that is independent of the test and the calibration data.
Then, the non-conformity score is defined as $R ( x , y ) = L ( M ( x ) , y )$.
For example, for regression problems, the model $M$ could be a deep neural network, and the loss $L$ the $l_{1}$ loss, resulting in the non-conformity score $R ( x , y ) = \size{M(x) - y}$.

The split conformal framework uses the above function to compute non-conformity scores.
For each calibration datum $Z_{i}$, for $i = 1 , \ldots , n$, we denote its corresponding calibration non-conformity score as $R_{i} = R ( Z_{i} )$.
Since the test datum label $Y_{n + 1}$ is not observed, we consider all possible realizations $y \in \mathcal{Y}$ of it and denote the corresponding test non-conformity score as $R ( X_{n + 1} , y )$, a function of the random variable $X_{n + 1}$ and the fixed variable $y$.
Subsequently, the framework computes an acceptance score threshold using only the calibration data; this is denoted by $\thr$ and is set to the $\lceil (1 - \alpha) (n + 1) \rceil$'th smallest value in the augmented set of calibration scores $\{ R_{1} , \ldots , R_{n} , \infty \}$.
Then, the label $y$ is included in the prediction set if its corresponding test non-conformity score is below this acceptance threshold, i.e.,~the split conformal prediction set is defined as,
\begin{equation}
    \ps = \left\{ y \in \mathcal{Y} : \thr \geq R ( X_{n + 1} , y ) \right\}.
\label{equation:conformal-set}
\end{equation}
These prediction sets satisfy~\cref{equation:conformal-error}~\citep{papadopoulos2002inductive,vovk2005algorithmic,lei2015conformal}.

Note that there is a naive way of satisfying the constraint in~\cref{equation:conformal-error}.
Consider the predictor that predicts the entire space of labels; this satisfies the constraint, but the predictions are uninformative.
Therefore,~\emph{the size of the prediction sets plays an important role in a conformal predictor's efficacy}---the smaller the size of the prediction sets, the better the predictor.

\section{RELATED WORK}

Since its original development~\citep{gammerman1998learning,saunders1999transduction}, many works have built on the basic conformal prediction framework.
Notably,~\citet{papadopoulos2002inductive,vovk2005algorithmic,lei2015conformal} proposed the aforementioned split conformal prediction as a special case.
We refer readers to~\citet{angelopoulos2023conformal} for a comprehensive overview.

The work done on conformal prediction set sizes has mostly focused on the statistical optimality of the family of predictors defined under this framework.
Such works use the expected size of the prediction sets as a notion of~\emph{inefficiency} for the predictors---the smaller the better.
They show that conformal predictors are~\emph{optimal} under different settings---such as unsupervised learning~\citep{lei2013distribution-sets,lei2015conformal}, regression~\citep{lei2013distribution-regression}, binary classification~\citep{lei2014classification}, and multi-class classification~\citep{sadinle2019least}---by showing that the expected size of their prediction sets asymptotically converges to that of an oracle.
Additionally,~\citet{vovk2014conformal,vovk2016criteria,sadinle2019least} provide similar optimality results when the probability of error is constrained conditionally per class/label.

This notion of viewing the expected size of the prediction sets as an inefficiency has propagated to practical settings as well, where conformal predictors are compared based on their average empirical prediction set sizes.
Furthermore, different non-conformity functions have been proposed to reduce this quantity.
For instance,~\citet{romano2019conformalized,kivaranovic2020adaptive} propose using quantile regression to train the machine learning model and an associated quantile interval loss function,~\citet{sadinle2019least,romano2020classification,angelopoulos2021uncertainty} propose loss functions for classification based on the predicted class/label probabilities, and~\citet{bellotti2021optimized,stutz2022learning} learn the non-conformity function in an end-to-end fashion by making the conformal pipeline differentiable.

\section{THEORETICAL QUANTIFICATION}
\label{section:size-quantification}

\begin{table*}[ht]
    \caption{
        \textbf{Multiplicative factor under different settings.}
        We summarize the multiplicative factor under different settings that we will utilize in~\cref{section:experiments} for the experiments.
        Most settings utilize $M ( x )$, the model prediction for input $x$; for LAC~\citep{sadinle2019least}, $M_{y} ( x )$ is the predicted probability for label $y$ and for CQR~\citep{romano2019conformalized}, $M_{\beta} ( x )$ is the prediction for the $\beta$'th level quantile and $M_{\Delta} ( x ) = ( M_{1 - \alpha / 2} ( x ) - M_{\alpha / 2} ( x ) ) / 2$.
        Details of these settings and the derivations of their multiplicative factors are provided in~\cref{appendix:multiplicative-factor}.
    }
    \label{table:multiplicative-factor}
    \vspace{-10pt}
    \begin{center}
    \begin{small}
    \begin{tabular}{lccc}
        \toprule
        \textbf{Problem type} & \textbf{Loss function} & \textbf{Non-conformity function} & \textbf{Multiplicative factor} \\
        & & $\pmb{R ( x , y )}$ & $\pmb{\mf}$ \\
        \midrule
        \multirow{4}{*}{\thead[l]{\small Regression \\ \small ($\mathcal{Y} = \mathds{R}$)}} & $l_{1}$ & $\size{M ( x ) - y}$ & $2$ \\
        & $l_{p \geq 1}$ & $\size{M ( x ) - y}^{p}$ & $2 r^{1 / p - 1} / p$ \\
        & \thead{\small CQR \\ \small \citep{romano2019conformalized}} & \thead{\small $\max \{ M_{\alpha / 2} ( x ) - y , \quad \quad$ \\ \small $\quad \quad y - M_{1 - \alpha / 2} ( x ) \}$} & $\begin{cases} 2 , & r \geq 0 \\ 2 \left( 1 - \cdf{M_{\Delta} ( X_{n + 1} )}{- r} \right) , & r < 0 \end{cases}$ \\
        \midrule
        \multirow{3}{*}{\thead[l]{\small Classification \\ \small (discrete $\mathcal{Y}$)}} & 0-1 & $\ind{M ( x ) \neq y}$ & $\begin{cases} 1 , & r = 0 \\ \size{\mathcal{Y}} - 1 , & r = 1 \end{cases}$ \\
        & \thead{\small LAC \\ \small \citep{sadinle2019least}} & $1 - M_{y} ( x )$ & $\sum_{y \in \mathcal{Y}} \pdf{M_{y} ( X_{n + 1} )}{1 - r}$ \\
        \bottomrule
    \end{tabular}
    \end{small}
    \end{center}
    \vspace{-10pt}
\end{table*}

We are interested in analyzing the size of prediction sets under the split conformal framework.
Similar to previous works, we will analyze the expected prediction set size $\mathds{E} [ \lvert \hat{C}_{\alpha}^{R} ( X_{n + 1} ; Z_{1 : n} ) \rvert ]$.\footnote{We overload $\lvert \cdot \rvert$ and $\int_{\mathcal{Y}} d y$ to be the Lebesgue measure for continuous and the counting measure for discrete spaces.}
While theoretical works considered the asymptotic behavior of this quantity, empirical works want to estimate it in practice.
We aim to bridge the two by theoretically quantifying the expected set size in the finite-sample case and deriving procedures to estimate it in practice, providing useful practical information about the prediction sets.

We begin with the quantification.
The prediction set size is the reference measure of the set of labels included.
However, from~\cref{equation:conformal-set} we know that the prediction set depends on:~(i) the test non-conformity scores $R ( X_{n + 1} , y )$, for all labels $y \in \mathcal{Y}$, (ii) the calibration non-conformity scores $R_{1 : n}$, and (iii) the significance level $\alpha$.
This implies that the prediction set is dependent on the test datum feature and the calibration data only through their corresponding non-conformity scores.
Therefore, instead of considering the label space, we analyze the space of non-conformity scores.

As a result, we compute the reference measure of the set of non-conformity scores below the acceptance score threshold $\thr$, as determined by the framework.
To translate this measure back to the label space, we introduce a multiplicative factor $\mf$, which we discuss in detail later.
With the multiplicative factor $\mf$ and the acceptance score threshold $\thr$, we show in~\cref{theorem:size-quantification} that the expected prediction set size is,
\begin{equation}
    \epss \!=\! \int_{\mathcal{R}} \pr{\thr \geq r} \mf d r,
\label{equation:size-quantification-non-iid}
\end{equation}
where the probability $\mathds{P} \{ \thr \geq r \}$ is over the calibration data $Z_{1 : n}$ and is not easy to compute.

To simplify it further, we assume that the calibration non-conformity scores are i.i.d.~from a probability distribution with the corresponding probability density/mass function $p_{R}$.\footnote{For a random variable $X$, we use $p_{X}$ and $P_{X}$ to denote the probability density/mass function and the cumulative distribution function respectively.}
Note that we are given i.i.d.~calibration data $Z_{1 : n}$, so their corresponding non-conformity scores $R_{1 : n}$ are i.i.d.~as well; here we define the distribution they follow.
With this, the individual probabilities of each calibration score being strictly smaller than $r$ are identical.
We denote this by $\tilde{P}_{R} ( r )$ and define it as,
\begin{equation}
    \tildecdf{R}{r} = \pr{R_{1} < r} = \int_{\mathcal{R}} \ind{r^{\prime} < r} \pdf{R}{r^{\prime}} d r^{\prime},
\label{equation:calibration-scores-tildecdf}
\end{equation}
noting that this is similar to, but not the same as, the cumulative distribution function $P_{R} ( r ) = \mathds{P} \{ R_{1} \leq r \}$.

As we are concerned with the event that the acceptance score threshold, i.e.,~the $\lceil (1 - \alpha) (n + 1) \rceil$'th smallest calibration score, is larger than $r$, we can allow at most $n_{\alpha} = \lceil (1 - \alpha) (n + 1) \rceil - 1$ calibration scores to be strictly less than $r$.
In doing so, we can now express $\mathds{P} \{ \thr \geq r \}$ as the cumulative distribution function of a binomial random variable.
The binomial random variable has $n$ trials and success probability $\tilde{P}_{R} ( r )$, which we denote by $B ( n , \tilde{P}_{R} ( r ) )$.
The required cumulative distribution function is evaluated at $n_{\alpha}$, and is denoted by $P_{B ( n , \tilde{P}_{R} ( r ) )} ( n_{\alpha} )$.
Therefore, the expected prediction set size in~\cref{equation:size-quantification-non-iid} simplifies to,
\begin{equation}
    \epss = \int_{\mathcal{R}} \cdf{\bin{n}{\tildecdf{R}{r}}}{n_{\alpha}} \mf d r.
\label{equation:size-quantification-iid}
\end{equation}

We package these two results together in the following theorem and provide the proof in~\cref{appendix:proof-size-quantification}.

\begin{theorem}[Expected size of prediction sets]

If the test and the calibration non-conformity scores are independent of each other, then the expected size of the split conformal prediction sets is given by~\cref{equation:size-quantification-non-iid}.
Furthermore, if the calibration non-conformity scores are i.i.d.,~then the expected size is given by~\cref{equation:size-quantification-iid}.

\label{theorem:size-quantification}
\end{theorem}

These results pertain to the marginal expected set size.
We include conditional expectations in~\cref{subsection:size-quantification-conditional}.

\paragraph{Multiplicative Factor}

The multiplicative factor is responsible for translating the reference measure on the space of non-conformity scores to the reference measure on the label space.
Formally, we define this factor as,
\begin{equation}
    \mf = \int_{\mathcal{Y}} \pdf{R ( X_{n + 1} , y )}{r} d y,
\label{equation:multiplicative-factor}
\end{equation}
where $p_{R ( X_{n + 1} , y )}$ is the probability density/mass function of the random variable $R ( X_{n + 1} , y )$, with the randomness from the random variable $X_{n + 1}$ (the test datum feature) and not the fixed variable $y$ (a label).

We provide derivations of the multiplicative factor under different settings in~\cref{appendix:multiplicative-factor} and summarize them in~\cref{table:multiplicative-factor}.
A common loss for regression is the $l_{1}$ loss, with the multiplicative factor 2.
We generalize to any $l_{p}$ loss ($p \geq 1$), with the multiplicative factor $2 r^{1 / p - 1} / p$; this highlights the reference measure translation that the multiplicative factor performs.
The 0-1 loss is a candidate for classification, with the multiplicative factor 1 and $\lvert \mathcal{Y} \vert - 1$ for $r = 0$ and $r = 1$ respectively.
Alternatively, more nuanced non-conformity functions have multiplicative factors that depend on the distribution of data and the machine learning model used.

For instance,~\citet{sadinle2019least} propose the least ambiguous set-valued classifiers (LAC) with the non-conformity function $R ( x , y ) = 1 - M_{y} ( x )$, where $M_{y} ( x )$ is the predicted probability for label $y$.
The associated multiplicative factor (cf.~\cref{table:multiplicative-factor}) cannot be analytically solved without making assumptions about the data distribution and/or the machine learning model.
However, LAC provably constructs prediction sets with minimum expected size if the predicted probabilities are correct; this does not hold in practice, but the predicted sets are small.
Similarly, conformalized quantile regression (CQR,~\citet{romano2019conformalized}) for regression and adaptive prediction sets (APS,~\citet{romano2020classification}) for classification construct small prediction sets, but their associated multiplicative factors are intractable without further assumptions.
This does not come as a surprise; the split conformal framework satisfies~\cref{equation:conformal-error}, but the quality of the prediction sets constructed depends on the data distribution and the machine learning model used~\citep{vovk2005algorithmic,shafer2008tutorial}.

We do not wish to make additional assumptions, so we treat the multiplicative factor associated with such non-conformity functions as unknown.
Note that~\cref{theorem:size-quantification} holds for any choice of the non-conformity function.

\subsection{Insights}
\label{subsection:size-quantification-insights}

After quantifying the expected prediction set size in~\cref{theorem:size-quantification}, we analyze its dependence on various user-specified parameters, providing general insights into influencing the quantity from a user's perspective.
We empirically validate our analysis in~\cref{appendix:experiments-uci-insights}.

\paragraph{Non-conformity Function}

The non-conformity function (constituting the machine learning model and the loss function) plays an important role.
Its influence on the set size is through the binomial random variable's success probability $\tilde{P}_{R} ( r )$ and the multiplicative factor $\mf$ (cf.~\cref{equation:size-quantification-iid}).
For example, for regression, $l_{1}$ and CQR~\citep{romano2019conformalized} use different machine learning models (with CQR using quantile regression models) and different loss functions ($l_{1}$ versus CQR loss).
For classification, 0-1, LAC~\citep{sadinle2019least}, and APS~\citep{romano2020classification} are different loss functions that could be used atop the same machine learning model.
Such modifications alter $\tilde{P}_{R} ( r )$ and $\mf$ and hence the expected set size (cf.~\cref{section:experiments}).

There are scenarios where the non-conformity function can change, but the multiplicative factor does not.
E.g.,~under the $l_{1}$ loss for regression, one can change the machine learning model, but $\mf = 2$ remains the same (cf.~\cref{table:multiplicative-factor}).
In such cases, only the influence through the binomial random variable's success probability $\tilde{P}_{R} ( r )$ matters.
Consider $\tilde{P}_{R_{1}}$ and $\tilde{P}_{R_{2}}$ corresponding to two such non-conformity functions $R_{1}$ and $R_{2}$ respectively, where the non-conformity score distribution of the first first-order stochastically dominates the second, i.e.,~$\tilde{P}_{R_{1}} ( r ) \leq \tilde{P}_{R_{2}} ( r )$, for all $r \in \mathcal{R}$.
Consequently, the expected set size is larger for the first function, $\mathds{E} [ \lvert \hat{C}^{R_{1}}_{\alpha} \left( X_{n + 1} ; Z_{1 : n} \right) \rvert ] \geq \mathds{E} [ \lvert \hat{C}^{R_{2}}_{\alpha} \left( X_{n + 1} ; Z_{1 : n} \right) \rvert ]$.
Therefore, for small expected set sizes, with the multiplicative factor being the same, $\tilde{P}_{R}$ should be skewed to have most of its probability density/mass on small values of $r \in \mathcal{R}$.
A common recipe to achieve this is by using a machine learning model that generalizes well.
As an analytical tool, a practitioner could also plot the empirical distribution of the calibration non-conformity scores to compare different non-conformity functions.

When the multiplicative factor changes, it is not straightforward to compare different non-conformity functions without making further assumptions.
This is especially the case when the space of non-conformity scores is modified in the process.
E.g.,~for regression, the $l_{1}$ loss admits non-negative non-conformity scores only, whereas CQR~\citep{romano2019conformalized} additionally permits negative scores.
If we were able to enforce and/or assume that CQR's scores smaller than $- c$ (for some constant $c \geq 0$) are never permissible, we could translate its space of non-conformity scores to the set of non-negative reals $\mathds{R}_{\geq 0}$ by adding an offset of $c$.
Then, offsetted CQR's multiplicative factor would never be larger than that of the $l_{1}$ loss (cf.~\cref{table:multiplicative-factor}); additionally, if the $l_{1}$ loss non-conformity score distribution first-order stochastically dominates that of offsetted CQR's scores, CQR would achieve a smaller expected set size.
Comparisons under other modifications of the multiplicative factor would require different assumptions.

\paragraph{Significance Level}

The significance level is generally fixed by a user rather than being tunable, specifying the user's requirement on the frequency of error.
However, we use this as a sanity check and highlight the trade-off between the error and the size of the prediction sets.
Intuitively, as the significance level increases, the framework allows for more errors which decreases the size of its prediction sets.
Indeed, in~\cref{equation:size-quantification-iid}, an increase in $\alpha$ causes a decrease in $n_{\alpha}$, which further prompts a decrease in the expected prediction set size.

\paragraph{Number of Calibration Data}

Labeled data procurement is often difficult, and a user might need justification for the benefits of collecting more data.
When using it for calibration, it is unclear how it would influence the expected prediction set size.
In~\cref{equation:size-quantification-iid}, an increase in $n$ causes an increase in both the number of trials of the binomial random variable and the value at which the cumulative distribution function of the said random variable is evaluated; the former decreases the expected set size while the latter increases it, diminishing their contributions.
Resolving this disagreement would require making more assumptions about the distribution of the calibration non-conformity scores.

\subsection{Conditional Expectation}
\label{subsection:size-quantification-conditional}

We quantified the marginal expected prediction set size in~\cref{theorem:size-quantification}, where we marginalize over the randomness induced by both the test datum feature and the calibration data.
For instance, when advertising its prediction system to potential customers, our company in~\cref{section:intro} computes the marginal expected set size as there is no additional information to condition on.
However, customers might be interested in evaluating the quality of the prediction sets constructed on a particular test input.
In this case, the marginal expected set size is not the quantity of interest, but the conditional expected set size conditioned on that test datum is.
We therefore extend~\cref{theorem:size-quantification} to allow for conditional expectations of the prediction set size.

When conditioning on the test datum feature $X_{n + 1} = x_{n + 1}$, there is additional information about the distribution of $R ( X_{n + 1} , y )$ and hence the multiplicative factor.
Synonymous with our definition of the multiplicative factor in~\cref{equation:multiplicative-factor}, we introduce the feature-specific multiplicative factor $\#_{R} ( r ; x_{n + 1} ) = \int_{\mathcal{Y}} \delta_{R ( x_{n + 1} , y )} ( r ) d y$, where the probability density/mass function $p_{R ( X_{n + 1} , y )}$ is replaced by $\delta_{R ( x_{n + 1} , y )}$, the Dirac delta distribution that places all of its probability mass on $R ( x_{n + 1} , y)$.
Then, the conditional expected set size is given by,
\begin{align}
    \begin{split}
        & \ev{\size{\ps} \middle\vert X_{n + 1} = x_{n + 1}} \\
        & = \int_{\mathcal{R}} \pr{\thr \geq r} \#_{R} \left( r ; x_{n + 1} \right) d r,
    \end{split}
    \label{equation:size-quantification-conditional-test-non-iid}
    \intertext{and if the calibration non-conformity scores are i.i.d.,}
    & = \int_{\mathcal{R}} \cdf{\bin{n}{\tildecdf{R}{r}}}{n_{\alpha}} \#_{R} \left( r ; x_{n + 1} \right) d r.
    \label{equation:size-quantification-conditional-test-iid}
\end{align}
Note that~\cref{equation:size-quantification-conditional-test-non-iid,equation:size-quantification-conditional-test-iid} are analogous to the marginals in~\cref{equation:size-quantification-non-iid,equation:size-quantification-iid} respectively.
We summarize this below (with the proof in~\cref{appendix:proof-size-quantification-conditional-test}).

\begin{corollary}[Expected size of prediction sets conditioned on the test datum feature]

If the test and the calibration non-conformity scores are independent of each other, then the expected size of the split conformal prediction sets conditioned on the test datum feature $X_{n + 1} = x_{n + 1}$ is given by~\cref{equation:size-quantification-conditional-test-non-iid}.
Furthermore, if the calibration non-conformity scores are i.i.d.,~then the conditional expected size is given by~\cref{equation:size-quantification-conditional-test-iid}.

\label{corollary:size-quantification-conditional-test}
\end{corollary}

We can similarly condition on the calibration data $Z_{1 : n} = z_{1 : n}$; we summarize the result in~\cref{corollary:size-quantification-conditional-calibration}.

\section{PRACTICAL ESTIMATION}
\label{section:size-estimation}

\begin{table*}[ht]
    \caption{
        \textbf{Practical estimates under different settings.}
        We summarize our point and interval estimates derived when the multiplicative factor is known (cf.~\cref{subsection:size-estimation-known-multiplicative-factor}) and when it is unknown (cf.~\cref{subsection:size-estimation-unknown-multiplicative-factor}).
    }
    \label{table:size-estimation}
    \vspace{-10pt}
    \begin{center}
    \begin{small}
    \begin{tabular}{lcc}
        \toprule
        \textbf{Setting} & \textbf{Our point estimate} & \textbf{Our interval estimate} (with significance $\gamma$) \\
        \midrule
        \thead[l]{\small Known \\ \small multiplicative factor} & $\int_{\mathcal{R}} \cdf{\bin{n}{\tildeempcdf{R}{r}}}{n_{\alpha}} \mf d r$ & \thead{\small $\left[ \int_{\mathcal{R}} \cdf{\bin{n}{\tildeempcdf{R}{r} + \Delta_{k , \gamma}}}{n_{\alpha}} \mf d r, \right. \quad \quad$ \\ \small $\quad \quad \left. \int_{\mathcal{R}} \cdf{\bin{n}{\tildeempcdf{R}{r} - \Delta_{k , \gamma}}}{n_{\alpha}} \mf d r \right]$} \\
        \midrule
        \thead[l]{\small Unknown \\ \small multiplicative factor} & $\int_{\mathcal{Y}} \frac{1}{k} \sum_{i = 1}^{k} \cdf{\bin{n}{\tildeempcdf{R}{R ( X^{\prime}_{i} , y )}}}{n_{\alpha}} d y$ & \thead{\small $\left[ \int_{\mathcal{Y}} \frac{1}{k} \sum_{i = 1}^{k} \cdf{\bin{n}{\tildeempcdf{R}{R ( X^{\prime}_{i} , y )} + \Delta_{k , \gamma}}}{n_{\alpha}} d y, \right. \quad \quad$ \\ \small $\quad \quad \left. \int_{\mathcal{Y}} \frac{1}{k} \sum_{i = 1}^{k} \cdf{\bin{n}{\tildeempcdf{R}{R ( X^{\prime}_{i} , y )} - \Delta_{k , \gamma}}}{n_{\alpha}} d y \right]$} \\
        \bottomrule
    \end{tabular}
    \end{small}
    \end{center}
    \vspace{-10pt}
\end{table*}

We theoretically quantify the expected prediction set size in~\cref{theorem:size-quantification}.
However, it assumes knowledge of the multiplicative factor $\mf$ and the binomial random variable's success probability $\tilde{P}_{R} ( r )$, for all non-conformity scores $r \in \mathcal{R}$.
While the former may be known under some settings (cf.~\cref{table:multiplicative-factor}), the latter is unknown in most practical scenarios as it relies on the distribution of the calibration non-conformity scores.

Currently, the expected set size is empirically estimated by averaging the size of the constructed prediction sets over multiple runs, i.e.,~a Monte Carlo average.
This equates to sampling a (pseudo) calibration data, obtaining conformal prediction sets on multiple (pseudo) test data, and repeating the process many times.
The average size of the obtained sets will estimate the expected set size; if repeated enough times, this estimate would be close to the true value.
For our company in~\cref{section:intro}, this involves collecting large amounts of labeled data from its customers and repeatedly executing the above procedure.
Furthermore, each such estimation scheme is instantiated with a fixed configuration of the significance level and the number of calibration data, resulting in an estimate that is configuration-specific.
Therefore, the company will need to carry out this Monte Carlo averaging scheme numerous times to obtain satisfactory estimates for varying values of the parameters.
This becomes infeasible in practice.

Alternatively, knowing the quantification of the expected prediction set size from~\cref{theorem:size-quantification}, we can develop procedures to estimate the value directly.
This will require data to be collected~\emph{only once}; we will assume access to $Z^{\prime}_{1} = ( X^{\prime}_{1} , Y^{\prime}_{1} ), \ldots, Z^{\prime}_{k} = ( X^{\prime}_{k} , Y^{\prime}_{k})$, $k$ data points drawn i.i.d.~from the data distribution.
Going back to our hypothetical company, possible ways of obtaining this data are either from a customer or held-out in-house company data.
We will detail procedures to derive point and interval estimates for the expected prediction set size using these accessible data points.
Our goals in doing so are:~(i) for the point estimate to be close to the expected set size, and (ii) for the interval to bound the expected set size with high probability.
We provide a summary in~\cref{table:size-estimation}.

Note that we consider the marginal expected set size, but our procedures can extend to the conditionals by substituting in the conditionally given quantities.

\subsection{Known Multiplicative Factor}
\label{subsection:size-estimation-known-multiplicative-factor}

We begin with the setting where the multiplicative factor can be analytically calculated and is known.
We compute the non-conformity scores for the $k$ accessible data points as $R^{\prime}_{i} = R ( Z^{\prime}_{i} )$, for $i = 1 , \ldots , k$.
We further use these non-conformity scores to empirically approximate $\tildecdf{R}{r}$, for all $r \in \mathcal{R}$, with the quantity,
\begin{equation}
    \tildeempcdf{R}{r} = \frac{1}{k} \sum\nolimits_{i = 1}^{k} \ind{R^{\prime}_{i} < r}.
\label{equation:calibration-score-tildecdf-estimate}
\end{equation}
By replacing $\tilde{P}_{R} ( r )$ with $\tilde{P}^{\text{emp}}_{R} ( r )$ in~\cref{equation:size-quantification-iid}, we obtain a point estimate for the expected set size.
$\tilde{P}^{\text{emp}}_{R}$ can also be used to estimate the expected set size under different significance levels and number of calibration data as $\tilde{P}_{R}$ is not dependent on these parameters.

We further provide guarantees for this estimate.
We use the work of~\citet{dvoretzky1956asymptotic,massart1990tight} to bound the difference between a cumulative distribution function and its empirical approximation.
Specifically, we can compute a $1 - \gamma$ confidence interval for $\tilde{P}_{R} ( r )$ of the form $[ \tilde{P}^{\text{emp}}_{R} ( r ) - \Delta_{k , \gamma}, \tilde{P}^{\text{emp}}_{R} ( r ) + \Delta_{k , \gamma}]$, where $\Delta_{k , \gamma} = \sqrt{\ln ( 2 / \gamma ) / 2 k}$.
Thus by replacing $\tilde{P}_{R} ( r )$ with $\tilde{P}^{\text{emp}}_{R} ( r ) \pm \Delta_{k , \gamma}$ in~\cref{equation:size-quantification-iid}, we obtain the lower-upper bounds corresponding to a $1 - \gamma$ confidence interval for the expected set size.
This is summarized in the following result
(with the proof in~\cref{appendix:proof-size-confidence-interval}).

\begin{corollary}[Confidence interval for the expected prediction set size]

Following~\cref{equation:size-quantification-iid} (\cref{theorem:size-quantification}), with a known multiplicative factor, the expected size of split conformal prediction sets lies in the interval,
\begin{equation}
\begin{split}
    \left[ \int_{\mathcal{R}} \cdf{\bin{n}{\tildeempcdf{R}{r} + \Delta_{k , \gamma}}}{n_{\alpha}} \mf d r, \right. \quad \quad \\
    \quad \quad \left. \int_{\mathcal{R}} \cdf{\bin{n}{\tildeempcdf{R}{r} - \Delta_{k , \gamma}}}{n_{\alpha}} \mf d r \right],
\end{split}
\label{equation:size-confidence-interval}
\end{equation}
with probability at least $1 - \gamma$.

\label{corollary:size-confidence-interval}
\end{corollary}

As $k$ increases, the error term $\Delta_{k , \gamma}$ decreases and so does the width of the confidence interval.
Therefore, the larger $k$ is, the tighter the confidence interval gets.
In fact, as $k \rightarrow \infty$, the error term $\Delta_{k, \gamma} \rightarrow 0$ and the confidence interval contracts to our point estimate.

\subsection{Unknown Multiplicative Factor}
\label{subsection:size-estimation-unknown-multiplicative-factor}

Next, we consider the setting where the multiplicative factor is intractable due to its dependence on the data distribution and/or the machine learning model.
One way to get around this is to estimate the factor using density estimation methods and substitute in its value.

To provide a self-contained approach, we re-arrange the formulation of the expected set size in~\cref{equation:size-quantification-iid} to get rid of the multiplicative factor (cf.~\cref{appendix:proof-size-estimation-unknown-multiplicative-factor}).
Instead, the quantification contains the expectation over the random variable $R ( X_{n + 1} , y )$ as follows,
\begin{equation}
\begin{split}
    & \epss \\
    & = \int_{\mathcal{Y}} \ev{\cdf{\bin{n}{\tildecdf{R}{R ( X_{n + 1} , y )}}}{n_{\alpha}}} d y.
\end{split}
\label{equation:size-quantification-iid-unknown-multiplicative-factor}
\end{equation}
The expectation $\mathds{E} [ P_{B ( n , \tilde{P}_{R} ( R ( X_{n + 1} , y ) ) )} ( n_{\alpha} ) ]$ contains two unknowns:~$\tilde{P}_{R}$ inside the expectation, and the distribution of $R ( X_{n + 1} , y )$ over which the expectation is evaluated.
This can be empirically approximated using nested Monte Carlo methods~\citep{rainforth2018nesting} with the accessible data points:~we approximate $\tilde{P}_{R}$ with $\tilde{P}^{\text{emp}}_{R}$ (cf.~\cref{equation:calibration-score-tildecdf-estimate}), and the distribution of $R ( X_{n + 1} , y )$ with the samples $R ( X^{\prime}_{i} , y )$, for $i = 1, \ldots, k$.
This amounts to the approximation $\frac{1}{k} \sum_{i = 1}^{k} P_{B ( n , \tilde{P}^{\text{emp}}_{R} ( R ( X^{\prime}_{i} , y ) ) )} ( n_{\alpha} )$ for the expectation term.
Integrating this quantity over $y \in \mathcal{Y}$ results in a point estimate for the expected set size, as desired.

Additionally, we can compute intervals synonymous with~\cref{equation:size-confidence-interval} by replacing $\tilde{P}^{\text{emp}}_{R} ( R ( X^{\prime}_{i} , y ) )$ with $\tilde{P}^{\text{emp}}_{R} ( R ( X^{\prime}_{i} , y ) ) \pm \Delta_{k , \gamma}$ above.
However, these may not be valid confidence intervals due to the extra approximation; we refer readers to~\citet{rainforth2018nesting} for nested Monte Carlo estimates' error analysis.
Despite this, we demonstrate their practical utility in~\cref{section:experiments}.

\section{EXPERIMENTS}
\label{section:experiments}

\begin{table*}[ht!]
    \caption{
        \textbf{Marginal expected prediction set sizes.}
        We illustrate the marginal expected sizes of split conformal prediction sets using different non-conformity functions and UCI datasets.
        The estimates are obtained via Monte Carlo averaging, our point estimates, and our interval estimates (lower-upper bounds with $\gamma = 0.1$).
        We also compute the absolute error between our individual point estimates and the mean Monte Carlo average.
        We report the means and standard deviations.
        For classification, the number of classes/labels is provided in parentheses.
    }
    \label{table:uci-size}
    \vspace{-10pt}
    \begin{center}
    \begin{scriptsize}
    \begin{tabular}{lllC{1.8cm}C{1.8cm}C{1.8cm}C{1.8cm}C{1.8cm}}
        \toprule
        & & & \multicolumn{4}{c}{\textbf{Marginal expected prediction set size}} & \\
        \cmidrule(lr){4-7}
        & & \textbf{Dataset} & \textbf{Our interval lower bound} & \textbf{Monte Carlo average} & \textbf{Our point estimate} & \textbf{Our interval upper bound} & \textbf{Absolute error} \\
        \midrule
        \multirow{20}{*}{\rotatebox[origin=c]{90}{\textbf{Regression ($\pmb{\mathcal{Y} = \mathds{R}}$)}}} & \multirow{10}{*}{\rotatebox[origin=c]{90}{\thead{$\pmb{l_{1}}$}}} & Abalone & 1.87\textsubscript{0.07} & 2.19\textsubscript{0.09} & 2.19\textsubscript{0.09} & 2.71\textsubscript{0.12} & 0.07\textsubscript{0.05} \\
        & & AirFoil & 1.11\textsubscript{0.07} & 1.39\textsubscript{0.10} & 1.39\textsubscript{0.09} & 2.03\textsubscript{0.13} & 0.08\textsubscript{0.05} \\
        & & AirQuality & 0.01\textsubscript{0.00} & 0.02\textsubscript{0.00} & 0.02\textsubscript{0.00} & 0.02\textsubscript{0.00} & 0.00\textsubscript{0.00} \\
        & & BlogFeedback & 2.30\textsubscript{0.02} & 2.38\textsubscript{0.02} & 2.38\textsubscript{0.02} & 2.47\textsubscript{0.02} & 0.02\textsubscript{0.01} \\
        & & CTSlices & 0.17\textsubscript{0.01} & 0.18\textsubscript{0.01} & 0.18\textsubscript{0.01} & 0.20\textsubscript{0.01} & 0.01\textsubscript{0.00} \\
        & & FacebookComments & 0.38\textsubscript{0.01} & 0.41\textsubscript{0.01} & 0.41\textsubscript{0.01} & 0.44\textsubscript{0.01} & 0.01\textsubscript{0.00} \\
        & & OnlineNews & 2.77\textsubscript{0.03} & 2.91\textsubscript{0.03} & 2.91\textsubscript{0.03} & 3.07\textsubscript{0.03} & 0.03\textsubscript{0.02} \\
        & & PowerPlant & 0.64\textsubscript{0.01} & 0.70\textsubscript{0.02} & 0.70\textsubscript{0.02} & 0.78\textsubscript{0.02} & 0.01\textsubscript{0.01} \\
        & & Superconductivity & 0.92\textsubscript{0.02} & 1.02\textsubscript{0.02} & 1.02\textsubscript{0.02} & 1.13\textsubscript{0.03} & 0.02\textsubscript{0.01} \\
        & & WhiteWineQuality & 2.29\textsubscript{0.06} & 2.58\textsubscript{0.08} & 2.58\textsubscript{0.07} & 2.99\textsubscript{0.09} & 0.06\textsubscript{0.05} \\
        \cmidrule(lr){2-8}
        & \multirow{10}{*}{\rotatebox[origin=c]{90}{\thead{\textbf{CQR} \\ \citep{romano2019conformalized}}}} & Abalone & 1.81\textsubscript{0.06} & 2.17\textsubscript{0.98} & 2.16\textsubscript{0.17} & 2.44\textsubscript{0.05} & 0.15\textsubscript{0.09} \\
        & & AirFoil & 1.36\textsubscript{0.05} & 1.58\textsubscript{0.64} & 1.58\textsubscript{0.07} & 2.09\textsubscript{0.13} & 0.05\textsubscript{0.04} \\
        & & AirQuality & 0.02\textsubscript{0.00} & 0.02\textsubscript{0.11} & 0.02\textsubscript{0.00} & 0.02\textsubscript{0.00} & 0.00\textsubscript{0.00} \\
        & & BlogFeedback & 1.39\textsubscript{0.01} & 1.39\textsubscript{0.96} & 1.39\textsubscript{0.01} & 1.39\textsubscript{0.01} & 0.01\textsubscript{0.01} \\
        & & CTSlices & 0.28\textsubscript{0.01} & 0.28\textsubscript{0.50} & 0.28\textsubscript{0.01} & 0.28\textsubscript{0.01} & 0.01\textsubscript{0.00} \\
        & & FacebookComments & 0.55\textsubscript{0.04} & 0.56\textsubscript{2.27} & 0.55\textsubscript{0.04} & 0.55\textsubscript{0.04} & 0.03\textsubscript{0.03} \\
        & & OnlineNews & 2.88\textsubscript{0.02} & 2.96\textsubscript{0.77} & 2.96\textsubscript{0.02} & 3.07\textsubscript{0.03} & 0.02\textsubscript{0.01} \\
        & & PowerPlant & 0.69\textsubscript{0.01} & 0.73\textsubscript{0.26} & 0.73\textsubscript{0.01} & 0.79\textsubscript{0.01} & 0.01\textsubscript{0.01} \\
        & & Superconductivity & 0.79\textsubscript{0.01} & 0.82\textsubscript{0.71} & 0.82\textsubscript{0.01} & 0.85\textsubscript{0.01} & 0.01\textsubscript{0.01} \\
        & & WhiteWineQuality & 2.24\textsubscript{0.10} & 2.24\textsubscript{0.89} & 2.24\textsubscript{0.10} & 2.27\textsubscript{0.10} & 0.07\textsubscript{0.06} \\
        \midrule
        \multirow{30}{*}{\rotatebox[origin=c]{90}{\textbf{Classification (discrete $\pmb{\mathcal{Y}}$)}}} & \multirow{10}{*}{\rotatebox[origin=c]{90}{\thead{\textbf{0-1}}}} & APSFailure (2) & 1.00\textsubscript{0.00} & 1.00\textsubscript{0.00} & 1.00\textsubscript{0.00} & 1.00\textsubscript{0.00} & 0.00\textsubscript{0.00} \\
        & & Adult (2) & 2.00\textsubscript{0.00} & 2.00\textsubscript{0.00} & 2.00\textsubscript{0.00} & 2.00\textsubscript{0.00} & 0.00\textsubscript{0.00} \\
        & & Avila (12) & 1.00\textsubscript{0.00} & 1.00\textsubscript{0.00} & 1.00\textsubscript{0.00} & 1.00\textsubscript{0.00} & 0.00\textsubscript{0.00} \\
        & & BankMarketing (2) & 1.00\textsubscript{0.00} & 1.00\textsubscript{0.00} & 1.01\textsubscript{0.02} & 1.74\textsubscript{0.23} & 0.01\textsubscript{0.02} \\
        & & CardDefault (2) & 2.00\textsubscript{0.00} & 2.00\textsubscript{0.00} & 2.00\textsubscript{0.00} & 2.00\textsubscript{0.00} & 0.00\textsubscript{0.00} \\
        & & Landsat (6) & 1.05\textsubscript{0.21} & 4.79\textsubscript{2.14} & 4.47\textsubscript{1.31} & 6.00\textsubscript{0.01} & 1.07\textsubscript{0.82} \\
        & & LetterRecognition (26) & 1.00\textsubscript{0.00} & 1.00\textsubscript{0.00} & 1.00\textsubscript{0.01} & 6.50\textsubscript{5.85} & 0.00\textsubscript{0.01} \\
        & & MagicGamma (2) & 1.97\textsubscript{0.06} & 2.00\textsubscript{0.00} & 2.00\textsubscript{0.00} & 2.00\textsubscript{0.00} & 0.00\textsubscript{0.00} \\
        & & SensorLessDrive (11) & 1.00\textsubscript{0.00} & 1.00\textsubscript{0.00} & 1.00\textsubscript{0.00} & 1.00\textsubscript{0.00} & 0.00\textsubscript{0.00} \\
        & & Shuttle (7) & 1.00\textsubscript{0.00} & 1.00\textsubscript{0.00} & 1.00\textsubscript{0.00} & 1.00\textsubscript{0.00} & 0.00\textsubscript{0.00} \\
        \cmidrule(lr){2-8}
        & \multirow{10}{*}{\rotatebox[origin=c]{90}{\thead{\textbf{LAC} \\ \citep{sadinle2019least}}}} & APSFailure (2) & 0.91\textsubscript{0.01} & 0.93\textsubscript{0.26} & 0.93\textsubscript{0.01} & 0.93\textsubscript{0.00} & 0.00\textsubscript{0.01} \\
        & & Adult (2) & 1.09\textsubscript{0.01} & 1.11\textsubscript{0.32} & 1.11\textsubscript{0.01} & 1.14\textsubscript{0.01} & 0.01\textsubscript{0.00} \\
        & & Avila (12) & 0.91\textsubscript{0.00} & 0.93\textsubscript{0.26} & 0.93\textsubscript{0.01} & 0.95\textsubscript{0.01} & 0.00\textsubscript{0.00} \\
        & & BankMarketing (2) & 0.96\textsubscript{0.00} & 0.99\textsubscript{0.12} & 0.99\textsubscript{0.00} & 1.01\textsubscript{0.01} & 0.00\textsubscript{0.00} \\
        & & CardDefault (2) & 1.20\textsubscript{0.01} & 1.25\textsubscript{0.44} & 1.25\textsubscript{0.01} & 1.32\textsubscript{0.01} & 0.01\textsubscript{0.01} \\
        & & Landsat (6) & 0.96\textsubscript{0.01} & 1.02\textsubscript{0.25} & 1.02\textsubscript{0.02} & 1.10\textsubscript{0.02} & 0.01\textsubscript{0.01} \\
        & & LetterRecognition (26) & 0.94\textsubscript{0.01} & 0.97\textsubscript{0.32} & 0.97\textsubscript{0.01} & 1.02\textsubscript{0.01} & 0.01\textsubscript{0.00} \\
        & & MagicGamma (2) & 1.03\textsubscript{0.01} & 1.07\textsubscript{0.26} & 1.07\textsubscript{0.01} & 1.12\textsubscript{0.01} & 0.01\textsubscript{0.01} \\
        & & SensorLessDrive (11) & 0.90\textsubscript{0.00} & 0.91\textsubscript{0.29} & 0.91\textsubscript{0.00} & 0.92\textsubscript{0.00} & 0.00\textsubscript{0.00} \\
        & & Shuttle (7) & 0.99\textsubscript{0.00} & 0.99\textsubscript{0.12} & 0.99\textsubscript{0.00} & 0.99\textsubscript{0.00} & 0.00\textsubscript{0.00} \\
        \cmidrule(lr){2-8}
        & \multirow{10}{*}{\rotatebox[origin=c]{90}{\thead{\textbf{APS} \\ \citep{romano2020classification}}}} & APSFailure (2) & 0.91\textsubscript{0.00} & 0.92\textsubscript{0.33} & 0.92\textsubscript{0.00} & 0.93\textsubscript{0.00} & 0.00\textsubscript{0.00} \\
        & & Adult (2) & 1.20\textsubscript{0.01} & 1.23\textsubscript{0.50} & 1.23\textsubscript{0.01} & 1.26\textsubscript{0.01} & 0.01\textsubscript{0.00} \\
        & & Avila (12) & 1.15\textsubscript{0.02} & 1.22\textsubscript{0.69} & 1.22\textsubscript{0.02} & 1.29\textsubscript{0.03} & 0.02\textsubscript{0.01} \\
        & & BankMarketing (2) & 1.07\textsubscript{0.01} & 1.09\textsubscript{0.46} & 1.09\textsubscript{0.01} & 1.11\textsubscript{0.01} & 0.01\textsubscript{0.00} \\
        & & CardDefault (2) & 1.30\textsubscript{0.01} & 1.36\textsubscript{0.50} & 1.36\textsubscript{0.01} & 1.42\textsubscript{0.01} & 0.01\textsubscript{0.01} \\
        & & Landsat (6) & 1.22\textsubscript{0.03} & 1.32\textsubscript{0.78} & 1.32\textsubscript{0.03} & 1.46\textsubscript{0.04} & 0.03\textsubscript{0.02} \\
        & & LetterRecognition (26) & 2.26\textsubscript{0.06} & 2.49\textsubscript{2.63} & 2.49\textsubscript{0.07} & 2.77\textsubscript{0.08} & 0.05\textsubscript{0.04} \\
        & & MagicGamma (2) & 1.16\textsubscript{0.01} & 1.21\textsubscript{0.49} & 1.21\textsubscript{0.01} & 1.27\textsubscript{0.02} & 0.01\textsubscript{0.01} \\
        & & SensorLessDrive (11) & 0.93\textsubscript{0.01} & 0.95\textsubscript{0.39} & 0.95\textsubscript{0.01} & 0.97\textsubscript{0.01} & 0.00\textsubscript{0.00} \\
        & & Shuttle (7) & 0.89\textsubscript{0.00} & 0.90\textsubscript{0.31} & 0.90\textsubscript{0.00} & 0.91\textsubscript{0.00} & 0.00\textsubscript{0.00} \\
        \bottomrule
    \end{tabular}
    \end{scriptsize}
    \end{center}
    \vspace{-10pt}
\end{table*}

\begin{figure*}[ht]
    \begin{center}
    \begin{subfigure}[b]{\textwidth}
        \includegraphics[width=\textwidth]{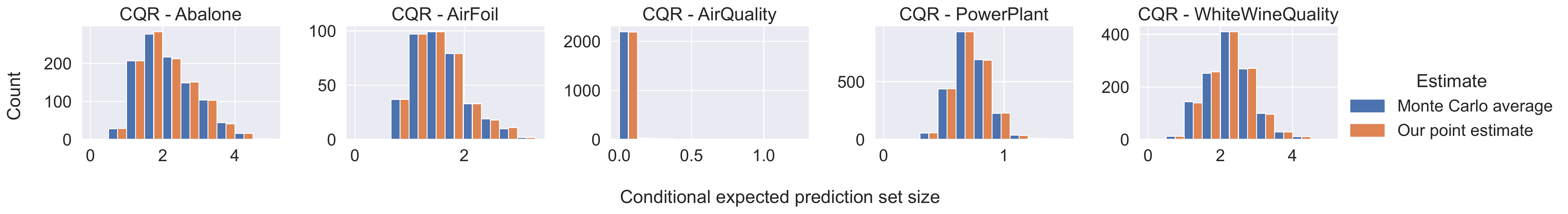}
    \end{subfigure}
    \begin{subfigure}[b]{\textwidth}
        \includegraphics[width=\textwidth]{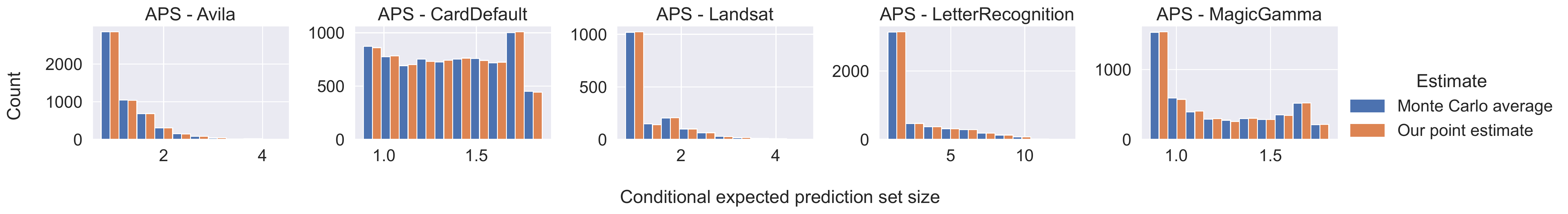}
    \end{subfigure}
    \end{center}
    \vspace{-10pt}
    \caption{
        \textbf{Expected prediction set sizes conditioned on the test datum feature.}
        We illustrate the expected sizes of split conformal prediction sets conditioned on varying test datum features using different UCI datasets.
        We use CQR~\citep{romano2019conformalized} for regression in the top row and APS~\citep{romano2020classification} for classification in the bottom row.
        The estimates are obtained via Monte Carlo averaging and our point estimates (refer to the legend for the color scheme); they are depicted as a histogram with side-by-side bars for comparison.
    }
    \label{figure:uci-test-conditional-size}
    \vspace{-10pt}
\end{figure*}

We now illustrate the efficacy of our results experimentally by applying our estimation procedures derived in~\cref{section:size-estimation} on real-world datasets from the UCI database~\citep{kelly2023uci}.
We use the $l_{1}$ loss and CQR~\citep{romano2019conformalized} non-conformity functions for regression, and the 0-1 loss, LAC~\citep{sadinle2019least}, and APS~\citep{romano2020classification} for classification.
We include the main experimental results here, with additional ones incorporated in~\cref{appendix:experiments-uci,appendix:experiments-synthetic}.
Our code is available at~\href{https://github.com/Guneet-Dhillon/expected-conformal-prediction-set-size}{https://github.com/Guneet-Dhillon/expected-conformal-prediction-set-size}.

\subsection{Experimental Setup}
\label{section:experiments-setup}

Similar to~\citet{tibshirani2019conformal}, we randomly split a dataset into 25\% training, 25\% calibration, and 50\% test.
We use the training data to train a random forest~\citep{breiman2001random} for the non-conformity function, utilizing the scikit-learn~\citep{pedregosa2011scikit} implementation with 100 trees.\footnote{For CQR~\citep{romano2019conformalized}, we train a quantile regression forest~\citep{meinshausen2006quantile} using the implementation from~\href{https://github.com/zillow/quantile-forest}{https://github.com/zillow/quantile-forest}.}
We run the split conformal algorithm on the calibration and the test data, with the significance level $\alpha$ set to 0.1.
We repeat this process 1000 times, where we sample new training data every 100 runs, and new calibration and test data every run.

\subsection{Marginal Expected Prediction Set Size}
\label{section:experiments-marginal}

We begin with the marginal expected size of prediction sets; we compare our derived estimates with the commonly used Monte Carlo averaging in approximating the marginal expected set size.
The latter is the average size of the prediction sets constructed using the split conformal algorithm on the test data using the calibration data, across all data splits.
For the former, we need accessible data points to make approximations.
While any data drawn i.i.d.~from the data distribution suffices, we use the calibration data as the accessible data to facilitate direct comparison of the two estimates:~Monte Carlo averaging uses both the test and the calibration data, whereas our estimates use only the calibration data.
Our estimation procedures provide both point and interval (with $\gamma = 0.1$) estimates; they are obtained from~\cref{subsection:size-estimation-known-multiplicative-factor} (with valid confidence intervals) for the $l_{1}$ and 0-1 loss non-conformity functions, and from~\cref{subsection:size-estimation-unknown-multiplicative-factor} for the other functions.

\cref{table:uci-size} illustrates the Monte Carlo average and our estimates for the marginal expected prediction set size.
The means of our point estimates are close to that of the Monte Carlo average, with the standard deviations being comparable or smaller despite using $3 \times$ fewer data points for the approximation.
This is also reflected in the low absolute error between our individual point estimates and the mean Monte Carlo average.
The only exception is when using the 0-1 loss non-conformity function on Landsat, but the standard deviations of the estimates are high under this setting.
Additionally, our interval estimates provide lower-upper bounds on the expected set size in practice (with the bounds being around our point estimate).
These results corroborate the efficacy of our derived practical estimates.

\subsection{Conditional Expected Prediction Set Size}
\label{section:experiments-conditional}

Next, we analyze the expected size of prediction sets conditioned on the test inputs.
Here, we consider the non-conformity functions CQR~\citep{romano2019conformalized} for regression and APS~\citep{romano2020classification} for classification, with other functions included in~\cref{appendix:experiments-uci-conditional}.

We follow a similar setup as before, but instead of using 50\% of the data as test, we use 25\% and fix them across the different data splits to compare the conditional expected prediction set sizes on these particular inputs; we will use the remaining 25\% as accessible data points.
As before, we compare our derived point estimate with Monte Carlo averaging.
In this case, the Monte Carlo average is the average size of the prediction sets constructed using the split conformal algorithm and the calibration data, across all data splits, on a fixed test datum.
On the other hand, we obtain our point estimate from~\cref{subsection:size-estimation-unknown-multiplicative-factor} using the accessible data (without having access to the calibration data), and condition on a fixed test datum feature with its feature-specific multiplicative factor (cf.~\cref{corollary:size-quantification-conditional-test}).

\cref{figure:uci-test-conditional-size} depicts histograms of the two estimates for the expected set sizes conditioned on varying test inputs.
The plots look identical for the two estimates, despite our point estimate not having seen the calibration data.
This further corroborates the efficacy of our estimates.

\section{CONCLUSIONS}

In this paper, we have studied the expected size of the prediction sets constructed by the split conformal framework.
We begin by theoretically quantifying the (marginal and conditional) expected prediction set size (cf.~\cref{section:size-quantification}).
Consequently, we derive practical estimation procedures that produce point estimates and high-probability interval bounds for the expected set size (cf.~\cref{section:size-estimation}); these procedures require data to be collected only once to produce reliable estimates.
Additionally, we corroborate our results experimentally on real-world regression and classification problems and demonstrate the efficacy of our estimates in practice.
Returning to our company and customers in~\cref{section:intro}, the company now has the tools to provide estimates of the expected set size, which allows potential customers to reliably evaluate the company's conformal system.

\subsubsection*{Acknowledgements}

Guneet S. Dhillon is supported by the Clarendon Fund Scholarship, University of Oxford.
Tom Rainforth is supported by the UK EPSRC grant EP/Y037200/1.

\bibliographystyle{abbrvnat}
\bibliography{main}

\newpage
\appendix
{
\onecolumn
\section{PROOFS}
\label{appendix:proofs}

\subsection{Proof for Theorem 1}
\label{appendix:proof-size-quantification}

\begin{proof}

The expected size of prediction sets under the split conformal prediction framework (cf.~\cref{equation:conformal-set}) is,
\begin{equation*}
\begin{split}
    \epss & = \ev{\size{\left\{ y \in \mathcal{Y} : \thr \geq R ( X_{n + 1} , y ) \right\}}} \\
    & = \ev{\int_{\mathcal{Y}} \ind{\thr \geq R ( X_{n + 1} , y )} d y} \\
    & = \int_{\mathcal{Y}} \pr{\thr \geq R ( X_{n + 1} , y )} d y \\
    & = \int_{\mathcal{Y}} \int_{\mathcal{R}} \pr{\thr \geq R ( X_{n + 1} , y ) \middle\vert R ( X_{n + 1} , y ) = r} \cdf{R ( X_{n + 1} , y )}{d r} d y,
\end{split}
\end{equation*}
where, for every label $y \in \mathcal{Y}$, we denote $P_{R ( X_{n + 1} , y )} ( d r )$ to be the law of the random variable $R ( X_{n + 1} , y )$, or equivalently, the push-forward of the marginal distribution of $X_{n + 1}$ under the mapping $X_{n + 1} \mapsto R ( X_{n + 1} , y )$.
In other words, $( y , A ) \in \mathcal{Y} \times \mathcal{B} ( \mathcal{R} ) \mapsto P_{R ( X_{n + 1} , y )} ( A )$ defines a transition kernel, where $\mathcal{B} ( \mathcal{R} )$ denotes the Borel $\sigma$-algebra of the space of non-conformity scores $\mathcal{R}$.
Continuing from above, we have that,
\begin{equation*}
\begin{split}
    \epss & = \int_{\mathcal{Y}} \int_{\mathcal{R}} \pr{\thr \geq R ( X_{n + 1} , y ) \middle\vert R ( X_{n + 1} , y ) = r} \cdf{R ( X_{n + 1} , y )}{d r} d y \\
    & \overset{(i)}{=} \int_{\mathcal{Y}} \int_{\mathcal{R}} \pr{\thr \geq r} \cdf{R ( X_{n + 1} , y )}{d r} d y \\
    & = \int_{\mathcal{R}} \pr{\thr \geq r} \int_{\mathcal{Y}} \cdf{R ( X_{n + 1} , y )}{d r} d y \\
    & = \int_{\mathcal{R}} \pr{\thr \geq r} \#_{R} ( d r ),
\end{split}
\end{equation*}
where $(i)$ follows from the test and the calibration non-conformity scores being independent of each other (since test and calibration data are independent of each other, so are their scores).
The measure $\#_{R}$ is defined as,
\begin{equation*}
    \#_{R} (A) = \int_{\mathcal{Y}} \cdf{R ( X_{n + 1} , y )}{A} d y,
\end{equation*}
for $A \in \mathcal{B} ( \mathcal{R} )$.
Note that $\#_{R} ( \mathcal{R} ) = \lvert \mathcal{Y} \rvert$, which may be infinite, for instance, when $\mathcal{Y} = \mathds{R}$.

If $\#_{R}$ is absolutely continuous w.r.t.~the reference measure, then $\#_{R} ( d r ) = \mf d r$.
If the law of $R ( X_{n + 1} , y )$ is absolutely continuous w.r.t.~the reference measure on $\mathcal{R}$, i.e.,~$P_{R ( X_{n + 1} , y )} ( d r ) = p_{R ( X_{n + 1} , y )} ( r ) d r$, then $\#_{R}$ is also absolutely continuous w.r.t.~the reference measure, with the following density,
\begin{equation*}
    \mf = \int_{\mathcal{Y}} \pdf{R ( X_{n + 1} , y )}{r} d y,
\end{equation*}
where we use the same symbol for the density.
This is the multiplicative factor, defined in~\cref{equation:multiplicative-factor}.

Continuing from above, we quantify the expected size of the prediction sets as follows,
\begin{equation*}
    \epss = \int_{\mathcal{R}} \pr{\thr \geq r} \mf d r,
\end{equation*}
which is the desired formulation in~\cref{equation:size-quantification-non-iid}.

Furthermore, the calibration non-conformity scores are i.i.d.~with the probability density/mass function $p_{R}$.
As a result, the individual (identical) probabilities for each calibration non-conformity score being strictly less than $r$ is $\tilde{P}_{R} ( r )$ (cf.~\cref{equation:calibration-scores-tildecdf}).
Additionally, the threshold $\thr$ is the $\lceil ( 1 - \alpha ) ( n + 1 ) \rceil$'th smallest value in the augmented set of calibration non-conformity scores $\{ R_{1} , \ldots , R_{n} , \infty \}$.
Then, for the event $\thr \geq r$ to occur, at most $n_{\alpha} = \lceil ( 1 - \alpha ) ( n + 1 ) \rceil - 1$ of the $n$ calibration non-conformity scores can be strictly smaller than $r$.
If we consider a calibration score being strictly smaller than $r$ as a success, we can simplify the event $\thr \geq r$ to the event $B ( n , \tilde{P}_{R} ( r ) ) \leq n_{\alpha}$, where $B ( n , \tilde{P}_{R} ( r ) )$ is a binomial random variable with $n$ trials and success probability $\tilde{P}_{R} ( r )$.
Finally, the probability of the event $\thr \geq r$ simplifies to the following,
\begin{equation*}
    \pr{\thr \geq r} = \pr{\bin{n}{\tildecdf{R}{r}} \leq n_{\alpha}} = \cdf{\bin{n}{\tildecdf{R}{r}}}{n_{\alpha}}.
\end{equation*}
Making the above simplification in~\cref{equation:size-quantification-non-iid} leads to the desired expected set size in~\cref{equation:size-quantification-iid}.

\end{proof}

\subsection{Proof for Corollary 2}
\label{appendix:proof-size-quantification-conditional-test}

\begin{proof}

The expected size of split conformal prediction sets conditioned on the test datum feature $X_{n + 1} = x_{n + 1}$ is,
\begin{equation*}
\begin{split}
    \ev{\size{\ps} \middle\vert X_{n + 1} = x_{n + 1}} & = \ev{\size{\left\{ y \in \mathcal{Y} : \thr \geq R ( X_{n + 1} , y ) \right\}} \middle\vert X_{n + 1} = x_{n + 1}} \\
    & = \ev{\int_{\mathcal{Y}} \ind{\thr \geq R ( X_{n + 1} , y )} d y \middle\vert X_{n + 1} = x_{n + 1}} \\
    & = \int_{\mathcal{Y}} \pr{\thr \geq R ( X_{n + 1} , y ) \middle\vert X_{n + 1} = x_{n + 1}} d y \\
    & \overset{(i)}{=} \int_{\mathcal{Y}} \pr{\thr \geq R ( x_{n + 1} , y )} d y \\
    & = \int_{\mathcal{Y}} \int_{\mathcal{R}} \pr{\thr \geq r} \delta_{R ( x_{n + 1} , y )} ( d r ) d y \\
    & = \int_{\mathcal{R}} \pr{\thr \geq r} \int_{\mathcal{Y}} \delta_{R ( x_{n + 1} , y )} ( d r ) d y \\
    & = \int_{\mathcal{R}} \pr{\thr \geq r} \#_{R} ( d r ; x_{n + 1} ),
\end{split}
\end{equation*}
where $(i)$ follows from the test and the calibration non-conformity scores being independent of each other (since the test and the calibration data are independent of each other).
The measure is $\#_{R} ( d r ; x_{n + 1} ) = \int_{\mathcal{Y}} \delta_{R ( x_{n + 1} , y )} ( d r ) d y$, where $\delta_{R ( x_{n + 1} , y )}$ is the Dirac delta distribution that places all of its probability mass on $R ( x_{n + 1} , y )$.
We define its Radon-Nikodym w.r.t.~the reference measure as $\#_{R} ( r ; x_{n + 1} ) = \int_{\mathcal{Y}} \delta_{R ( x_{n + 1} , y )} ( r ) d y$, when it exists; this is the feature-specific multiplicative factor (cf.~\cref{subsection:size-quantification-conditional}).
As a result, we obtain $\#_{R} ( d r ; x_{n + 1} ) = \#_{R} ( r ; x_{n + 1} ) d r$.

Continuing, we quantify the expected prediction set size conditioned on the test datum feature $X_{n + 1} = x_{n + 1}$ as,
\begin{equation*}
    \ev{\size{\ps} \middle\vert X_{n + 1} = x_{n + 1}} = \int_{\mathcal{R}} \pr{\thr \geq r} \#_{R} ( r ; x_{n + 1} ) d r,
\end{equation*}
which is the desired formulation in~\cref{equation:size-quantification-conditional-test-non-iid}.
Furthermore, following the proof in~\cref{appendix:proof-size-quantification}, we can simplify $\mathds{P} \{ \thr \geq r \} = P_{B ( n , \tildecdf{R}{r} )} ( n_{\alpha} )$, to get the desired conditional expected prediction set size in~\cref{equation:size-quantification-conditional-test-iid}.

\end{proof}

\subsection{Proof for Corollary 3}
\label{appendix:proof-size-confidence-interval}

\begin{proof}

For each calibration non-conformity score $R_{i}$, for $i = 1, \ldots, n$, we introduce a new random variable $V_{i} = - R_{i}$.
Since $R_{1 : n}$ are i.i.d.,~the random variables $V_{1 : n}$ are i.i.d.~as well, and we denote their cumulative distribution function as $P_{V}$.
We recognize that $\tilde{P}_{R}$ (cf.~\cref{equation:calibration-scores-tildecdf}) and $P_{V}$ are related in the following way,
\begin{equation*}
    \tildecdf{R}{r} = \pr{R_{1} < r} = \pr{- R_{1} > - r} = \pr{V_{1} > - r} = 1 - \pr{V_{1} \leq - r} = 1 - \cdf{V}{- r}.
\end{equation*}
Equivalently, the empirical approximation $\tilde{P}^{\text{emp}}_{R}$ (cf.~\cref{equation:calibration-score-tildecdf-estimate}) approximates $P_{V}$ in the following way,
\begin{equation*}
\begin{split}
    & \tildeempcdf{R}{r} = \frac{1}{k} \sum\nolimits_{i = 1}^{k} \ind{R^{\prime}_{i} < r} = \frac{1}{k} \sum\nolimits_{i = 1}^{k} \ind{- R^{\prime}_{i} > - r} \\
    & = \frac{1}{k} \sum\nolimits_{i = 1}^{k} \ind{V^{\prime}_{i} > - r} = 1 - \frac{1}{k} \sum\nolimits_{i = 1}^{k} \ind{V^{\prime}_{i} \leq - r} = 1 - \empcdf{V}{- r},
\end{split}
\end{equation*}
where we define $V^{\prime}_{i} = - R^{\prime}_{i}$, for $i = 1, \ldots, k$.

The Dvoretzky–Kiefer–Wolfowitz inequality~\citep{dvoretzky1956asymptotic,massart1990tight} bounds the difference between the cumulative distribution function and its empirical approximation, which can be transformed into confidence intervals.
We set $\Delta_{k , \gamma} = \sqrt{\ln ( 2 / \gamma ) / 2 k}$.
With probability at least $1 - \gamma$, for all $r \in \mathcal{R}$,
\begin{equation*}
\begin{split}
    & \cdf{V}{- r} \in \left[ \empcdf{V}{- r} - \Delta_{k , \gamma}, \empcdf{V}{- r} + \Delta_{k , \gamma} \right] \\
    & \iff 1 - \cdf{V}{- r} \in \left[ 1 - \empcdf{V}{- r} - \Delta_{k , \gamma}, 1 - \empcdf{V}{- r} + \Delta_{k , \gamma} \right] \\
    & \iff \tildecdf{R}{r} \in \left[ \tildeempcdf{R}{r} - \Delta_{k , \gamma}, \tildeempcdf{R}{r} + \Delta_{k , \gamma} \right].
\end{split}
\end{equation*}
This implies that with probability at least $1 - \gamma$, for all $r \in \mathcal{R}$,
\begin{equation*}
    \cdf{\bin{n}{\tildecdf{R}{r}}}{n_{\alpha}} \in \left\{ \cdf{\bin{n}{p}}{n_{\alpha}} \right\}_{p \in \left[ \tildeempcdf{R}{r} - \Delta_{k , \gamma}, \tildeempcdf{R}{r} + \Delta_{k , \gamma} \right]}.
\end{equation*}
Since $P_{B ( n , p )} ( n_{\alpha} )$ is a non-increasing function in $p$, with probability at least $1 - \gamma$, for all $r \in \mathcal{R}$,
\begin{equation*}
    \cdf{\bin{n}{\tildecdf{R}{r}}}{n_{\alpha}} \in \left[ \cdf{\bin{n}{\tildeempcdf{R}{r} + \Delta_{k , \gamma}}}{n_{\alpha}}, \cdf{\bin{n}{\tildeempcdf{R}{r} - \Delta_{k , \gamma}}}{n_{\alpha}} \right].
\end{equation*}
Since this holds for all $r \in \mathcal{R}$ simultaneously, with probability at least $1 - \gamma$,
\begin{equation*}
\begin{gathered}
    \epss \overset{(i)}{=} \int_{\mathcal{R}} \cdf{\bin{n}{\tildecdf{R}{r}}}{n_{\alpha}} \mf d r \in \\
    \left[ \int_{\mathcal{R}} \cdf{\bin{n}{\tildeempcdf{R}{r} + \Delta_{k , \gamma}}}{n_{\alpha}} \mf d r, \int_{\mathcal{R}} \cdf{\bin{n}{\tildeempcdf{R}{r} - \Delta_{k , \gamma}}}{n_{\alpha}} \mf d r \right],
\end{gathered}
\end{equation*}
where $(i)$ follows from \cref{equation:size-quantification-iid} and the multiplicative factor $\mf$ is known.
This is the desired confidence interval for the expected size of the prediction sets in~\cref{equation:size-confidence-interval} of~\cref{corollary:size-confidence-interval}.

\end{proof}

\subsection{Corollary 4}
\label{appendix:proof-size-quantification-conditional-calibration}

In addition to~\cref{corollary:size-quantification-conditional-test} where we quantify the expected prediction set size conditioned on a test datum feature, we can condition on the calibration data $Z_{1 : n} = z_{1 : n}$ instead.
We summarize the corresponding result below.

\begin{corollary}[Expected size of prediction sets conditioned on the calibration data]

If the test and the calibration non-conformity scores are independent of each other, then the expected size of the split conformal prediction sets conditioned on the calibration data $Z_{1 : n} = z_{1 : n}$ (and $r_{i} = R( z_{i} )$, for $i = 1 , \ldots , n$) is given by the following,
\begin{equation}
    \ev{\size{\ps} \middle\vert Z_{1 : n} = z_{1 : n}} = \int_{\mathcal{R}} \ind{\tau_{\alpha} \left( r_{1 : n} \right) \geq r} \mf d r.
\label{equation:size-quantification-conditional-calibration}
\end{equation}

\label{corollary:size-quantification-conditional-calibration}
\end{corollary}

\begin{proof}

The expected size of split conformal prediction sets conditioned on the calibration data $Z_{1 : n} = z_{1 : n}$ is,
\begin{equation*}
\begin{split}
    \ev{\size{\ps} \middle\vert Z_{1 : n} = z_{1 : n}} & = \ev{\size{\left\{ y \in \mathcal{Y} : \thr \geq R ( X_{n + 1} , y ) \right\}} \middle\vert Z_{1 : n} = z_{1 : n}} \\
    & = \ev{\int_{\mathcal{Y}} \ind{\thr \geq R ( X_{n + 1} , y )} d y \middle\vert Z_{1 : n} = z_{1 : n}} \\
    & = \int_{\mathcal{Y}} \pr{\thr \geq R ( X_{n + 1} , y ) \middle\vert Z_{1 : n} = z_{1 : n}} d y \\
    & \overset{(i)}{=} \int_{\mathcal{Y}} \pr{\tau_{\alpha} \left( r_{1 : n} \right) \geq R ( X_{n + 1} , y )} d y \\
    & \overset{(ii)}{=} \int_{\mathcal{R}} \ind{\tau_{\alpha} \left( r_{1 : n} \right) \geq r} \mf d r,
\end{split}
\end{equation*}
where $(i)$ follows from the test and the calibration non-conformity scores being independent of each other (since the test and the calibration data are independent of each other), and $(ii)$ follows from the proof in~\cref{appendix:proof-size-quantification}.

\end{proof}

\subsection{Proof for Practical Estimation (Unknown Multiplicative Factor)}
\label{appendix:proof-size-estimation-unknown-multiplicative-factor}

Under the setting where the multiplicative factor is intractable due to its dependence on the data distribution and/or the machine learning model, we can re-arrange the formulation of the expected prediction set size in~\cref{equation:size-quantification-iid} to get rid of this factor.
Using~\cref{equation:size-quantification-iid,equation:multiplicative-factor}, we obtain the following,
\begin{equation*}
\begin{split}
    & \epss = \int_{\mathcal{R}} \cdf{\bin{n}{\tildecdf{R}{r}}}{n_{\alpha}} \mf d r = \int_{\mathcal{R}} \cdf{\bin{n}{\tildecdf{R}{r}}}{n_{\alpha}} \int_{\mathcal{Y}} \pdf{R ( X_{n + 1} , y )}{r} d y \, d r \\
    & = \int_{\mathcal{Y}} \int_{\mathcal{R}} \cdf{\bin{n}{\tildecdf{R}{r}}}{n_{\alpha}} \pdf{R ( X_{n + 1} , y )}{r} d r \, d y = \int_{\mathcal{Y}} \ev{\cdf{\bin{n}{\tildecdf{R}{R ( X_{n + 1} , y )}}}{n_{\alpha}}} d y,
\end{split}
\end{equation*}
where the expectation term is evaluated over the random variable $R ( X_{n + 1} , y )$.
This derives~\cref{equation:size-quantification-iid-unknown-multiplicative-factor}.

\section{MULTIPLICATIVE FACTOR}
\label{appendix:multiplicative-factor}

The multiplicative factor (cf.~\cref{equation:multiplicative-factor}) is responsible for translating the reference measure on the space of non-conformity scores to the reference measure on the label space.
Here we provide the derivation of this factor under different settings used in~\cref{section:experiments} for the experiments.
Note that we overload $\int_{\mathcal{Y}} d y$ to be the Lebesgue measure when $\mathcal{Y}$ is continuous and the counting measure when it is discrete (amounting to a sum over $\mathcal{Y}$).

\subsection{\textit{l}\textit{\scriptsize p} Loss for Regression}
\label{appendix:multiplicative-factor-lp-regression}

We begin with regression problems, where the label space is the set of reals, i.e.,~$\mathcal{Y} = \mathds{R}$.
A common non-conformity function for such problems involves a machine learning model $M$ and the $l_{1}$ loss~\citep{papadopoulos2002inductive,vovk2005algorithmic,tibshirani2019conformal,barber2023conformal}.
We generalize this to any $l_{p}$ loss, with $p \geq 1$.
Then, the non-conformity function is given by $R ( x , y ) = \lvert M ( x ) - y \rvert^{p}$ and the space of non-conformity scores is $\mathcal{R} = [ 0 , \infty )$.

We can define the cumulative distribution function of $R ( X_{n + 1} , y )$ as follows,
\begin{equation*}
\begin{split}
    & \cdf{R ( X_{n + 1} , y )}{r} = \pr{R ( X_{n + 1} , y ) \leq r} = \pr{\size{M ( X_{n + 1} ) - y}^{p} \leq r} = \pr{\size{M ( X_{n + 1} ) - y} \leq r^{1 / p}} \\
    & = \pr{y - r^{1 / p} \leq M ( X_{n + 1} ) \leq y + r^{1 / p}} = \pr{M ( X_{n + 1} ) \leq y + r^{1 / p}} - \pr{M ( X_{n + 1} ) < y - r^{1 / p}} \\
    & \overset{(i)}{=} \pr{M ( X_{n + 1} ) \leq y + r^{1 / p}} - \pr{M ( X_{n + 1} ) \leq y - r^{1 / p}} = \cdf{M ( X_{n + 1} )}{y + r^{1 / p}} - \cdf{M ( X_{n + 1} )}{y - r^{1 / p}},
\end{split}
\end{equation*}
where $(i)$ follows from the prediction $M ( X_{n + 1} )$ being continuous.
Further, differentiating the above with respect to $r$, we get the probability density function of $R ( X_{n + 1} , y )$ as follows,
\begin{equation*}
    \pdf{R ( X_{n + 1} , y )}{r} = \frac{r^{1 / p - 1}}{p} \pdf{M ( X_{n + 1} )}{y + r^{1 / p}} + \frac{r^{1 / p - 1}}{p} \pdf{M ( X_{n + 1} )}{y - r^{1 / p}}.
\end{equation*}
Therefore, the multiplicative factor in this setting is given by,
\begin{equation*}
\begin{split}
    & \mf = \int_{\mathds{R}} \pdf{R ( X_{n + 1} , y )}{r} d y = \int_{\mathds{R}} \frac{r^{1 / p - 1}}{p} \pdf{M ( X_{n + 1} )}{y + r^{1 / p}} d y + \int_{\mathds{R}} \frac{r^{1 / p - 1}}{p} \pdf{M ( X_{n + 1} )}{y - r^{1 / p}} d y \\
    & = \frac{r^{1 / p - 1}}{p} \left( \int_{\mathds{R}} \pdf{M ( X_{n + 1} )}{y + r^{1 / p}} d y + \int_{\mathds{R}} \pdf{M ( X_{n + 1} )}{y - r^{1 / p}} d y \right) \\
    & \overset{(ii)}{=} \frac{r^{1 / p - 1}}{p} \left( \int_{\mathds{R}} \pdf{M ( X_{n + 1} )}{u} d u + \int_{\mathds{R}} \pdf{M(X_{n + 1})}{v} d v \right) = \frac{r^{1 / p - 1}}{p} ( 1 + 1 ) = \frac{2 r^{1 / p - 1}}{p},
\end{split}
\end{equation*}
where $(ii)$ follows from a change of variables with $u = y + r^{1 / p}$ and $v = y - r^{1 / p}$.

\paragraph{\textit{l}\textit{\scriptsize p} Loss for High-Dimensional Regression}

We do not restrict ourselves to one-dimensional regression; we can further generalize the above to $m$-dimensional regression problems, where the label space is $\mathcal{Y} = \mathds{R}^{m}$, with $m \geq 1$.
The non-conformity function involves a machine learning model $M$ and the $l_{p}$ loss, with $p \geq 1$.
Then, the non-conformity function is given by $R ( x , y ) = \lVert M ( x ) - y \rVert^{p}_{p}$ and the space of non-conformity scores is $\mathcal{R} = [ 0 , \infty )$.

The multiplicative factor in this setting is given by,
\begin{equation*}
\begin{split}
    & \mf = \int_{\mathds{R}^{m}} \pdf{R ( X_{n + 1} , y )}{r} d y = \int_{\mathds{R}^{m}} \frac{d}{d r} \cdf{R ( X_{n + 1} , y )}{r} d y = \frac{d}{d r} \int_{\mathds{R}^{m}} \cdf{R ( X_{n + 1} , y )}{r} d y \\
    & = \frac{d}{d r} \int_{\mathds{R}^{m}} \pr{R ( X_{n + 1} , y ) \leq r} d y = \frac{d}{d r} \int_{\mathds{R}^{m}} \pr{\left\lVert M ( X_{n + 1} ) - y \right\rVert^{p}_{p} \leq r} d y \\
    & = \frac{d}{d r} \int_{\mathds{R}^{m}} \pr{\left\lVert M ( X_{n + 1} ) - y \right\rVert_{p} \leq r^{1 / p}} d y.
\end{split}
\end{equation*}
We denote $B^{p}_{m} ( c , r )$ to be a $m$-dimensional $l_{p}$-ball with center $c$ and radius $r$, and $V^{p}_{m} ( r )$ to be the volume of a $m$-dimensional $l_{p}$-ball with radius $r$.
Continuing from above, we have that,
\begin{equation*}
\begin{split}
    & \mf = \frac{d}{d r} \int_{\mathds{R}^{m}} \pr{\left\lVert M ( X_{n + 1} ) - y \right\rVert_{p} \leq r^{1 / p}} d y = \frac{d}{d r} \int_{\mathds{R}^{m}} \pr{y \in B^{p}_{m} \left( M ( X_{n + 1} ) , r^{1 / p} \right)} d y \\
    & = \frac{d}{d r} \int_{\mathds{R}^{m}} \int_{\mathds{R}^{m}} \pr{y \in B^{p}_{m} \left( M ( X_{n + 1} ) , r^{1 / p} \right) \middle\vert M ( X_{n + 1} ) = c} \pdf{M ( X_{n + 1} )}{c} d c \, d y \\
    & = \frac{d}{d r} \int_{\mathds{R}^{m}} \int_{\mathds{R}^{m}} \ind{y \in B^{p}_{m} \left( c , r^{1 / p} \right)} \pdf{M ( X_{n + 1} )}{c} d c \, d y = \frac{d}{d r} \int_{\mathds{R}^{m}} \pdf{M ( X_{n + 1} )}{c} \int_{\mathds{R}^{m}} \ind{y \in B^{p}_{m} \left( c , r^{1 / p} \right)} d y \, d c \\
    & = \frac{d}{d r} \int_{\mathds{R}^{m}} \pdf{M ( X_{n + 1} )}{c} V^{p}_{m} \left( r^{1 / p} \right) d c = \frac{d}{d r} V^{p}_{m} \left( r^{1 / p} \right) \int_{\mathds{R}^{m}} \pdf{M ( X_{n + 1} )}{c} d c = \frac{d}{d r} V^{p}_{m} \left( r^{1 / p} \right).
\end{split}
\end{equation*}
Furthermore, the volume of a $m$-dimensional $l_{p}$-ball with radius $r$ is $V^{p}_{m} ( r ) = ( 2 \Gamma ( 1 / p + 1 ) )^{m} r^{m} / \Gamma ( m / p + 1 )$.
Therefore, the multiplicative factor in this setting is given by,
\begin{equation*}
    \mf = \frac{d}{d r} V^{p}_{m} \left( r^{1 / p} \right) = \frac{d}{d r} \frac{\left( 2 \Gamma ( 1 / p + 1 ) \right)^{m}}{\Gamma ( m / p + 1 )} r^{m / p} = \frac{\left( 2 \Gamma ( 1 / p + 1 ) \right)^{m}}{\Gamma ( m / p + 1 )} \frac{m}{p} r^{m / p - 1}.
\end{equation*}
Note that this is a generalization of our previous result for 1-dimensional regression ($\mathcal{Y} = \mathds{R}$) to higher dimensions ($\mathcal{Y} = \mathds{R}^{m}$).
Indeed, by substituting $m = 1$, we recover the multiplicative factor $\mf = 2 r^{1 / p - 1} / p$ as before.

\subsection{0-1 Loss for Classification}
\label{appendix:multiplicative-factor-zeroone-classification}

Here we consider classification problems, where the label space is discrete.
The machine learning model $M$ predicts a label directly, i.e.,~$M ( x ) \in \mathcal{Y}$ for an input feature $x \in \mathcal{X}$.
We consider the 0-1 loss which takes the value 0 if the prediction is correct and 1 when incorrect.
Then, the non-conformity function is given by $R ( x , y ) = \mathds{1} \{ M ( x ) \neq y \}$ and the space of non-conformity scores is $\mathcal{R} = \{ 0 , 1 \}$.
The multiplicative factor in this setting is given by,
\begin{equation*}
\begin{split}
    & \mf = \sum_{y \in \mathcal{Y}} \pdf{R ( X_{n + 1} , y )}{r} = \sum_{y \in \mathcal{Y}} \pr{R ( X_{n + 1} , y ) = r} = \sum_{y \in \mathcal{Y}} \pr{\ind{M ( x ) \neq y} = r} \\
    & = \begin{cases} \sum_{y \in \mathcal{Y}} \pr{M ( x ) = y} , & r = 0 \\ \sum_{y \in \mathcal{Y}} \pr{M ( x ) \neq y} , & r = 1 \end{cases} = \begin{cases} \sum_{y \in \mathcal{Y}} \pr{M ( x ) = y} , & r = 0 \\ \sum_{y \in \mathcal{Y}} \left( 1 - \pr{M ( x ) = y} \right) , & r = 1 \end{cases} = \begin{cases} 1 , & r = 0 \\ \size{\mathcal{Y}} - 1 , & r = 1 \end{cases}.
\end{split}
\end{equation*}

\subsection{Other Settings}
\label{appendix:multiplicative-factor-other}

There are many other non-conformity functions proposed for regression and classification problems.
However, in some cases, the multiplicative factor can depend on the distribution of data and the machine learning model used.

\paragraph{Least Ambiguous Set-Valued Classifiers (LAC)}

\citet{sadinle2019least} propose LAC that provably construct prediction sets with minimum expected size if the predicted probabilities are correct; this does not hold in practice, but the predicted sets are small.
In this case, the machine learning model $M$ predicts a probability distribution over the labels; we denote $M_{y} (x)$ as the predicted probability for label $y \in \mathcal{Y}$ for an input feature $x \in \mathcal{X}$.
The non-conformity function is given by $R ( x , y ) = 1 - M_{y} ( x )$ and the space of non-conformity scores is $\mathcal{R} = [ 0 , 1 ]$.

We can define the cumulative distribution function of $R ( X_{n + 1} , y )$ as follows,
\begin{equation*}
\begin{split}
    & \cdf{R ( X_{n + 1} , y )}{r} = \pr{R ( X_{n + 1} , y ) \leq r} = \pr{1 - M_{y} ( X_{n + 1} ) \leq r} = \pr{M_{y} ( X_{n + 1} ) \geq 1 - r} \\
    & = 1 - \pr{M_{y} ( X_{n + 1} ) < 1 - r} \overset{(i)}{=} 1 - \pr{M_{y} ( X_{n + 1} ) \leq 1 - r} = 1 - \cdf{M_{y} ( X_{n + 1} )}{1 - r},
\end{split}
\end{equation*}
where $(i)$ follows from the prediction $M_{y} ( X_{n + 1} )$ being continuous.
Further, differentiating the above with respect to $r$, we get the probability density function of $R ( X_{n + 1} , y )$ as follows,
\begin{equation*}
    \pdf{R ( X_{n + 1} , y )}{r} = \pdf{M_{y} ( X_{n + 1} )}{1 - r}.
\end{equation*}
Therefore, the multiplicative factor in this setting is given by,
\begin{equation*}
    \mf = \sum_{y \in \mathcal{Y}} \pdf{R ( X_{n + 1} , y )}{r} = \sum_{y \in \mathcal{Y}} \pdf{M_{y} ( X_{n + 1} )}{1 - r},
\end{equation*}
which is dependent on the data distribution and the machine learning model.
Consequently, the multiplicative factor cannot be analytically solved under this setting without making any further assumptions.

\paragraph{Conformalized Quantile Regression (CQR)}

\citet{romano2019conformalized} propose a non-conformity function for regression ($\mathcal{Y} = \mathds{R}$).
In this case, the machine learning model is trained using quantile regression~\citep{koenker1978regression} to have two outputs $M_{\alpha / 2} ( x ) , M_{1 - \alpha / 2} ( x ) \in \mathds{R}$, corresponding to predictions of the $(\alpha / 2)$'th and $(1 - \alpha / 2)$'th level quantiles respectively, conditioned on an input feature $x \in \mathcal{X}$.
Further, the proposed non-conformity function is a loss on the predicted quantile interval, given by $R ( x , y ) = \max \{ M_{\alpha / 2} ( x ) - y , y - M_{1 - \alpha / 2} ( x ) \}$.
For ease of notation when deriving the associated multiplicative factor, we set $M ( x ) = ( M_{1 - \alpha / 2} ( x ) + M_{\alpha / 2} ( x ) ) / 2 \in \mathds{R}$ and $M_{\Delta} ( x ) = ( M_{1 - \alpha / 2} ( x ) - M_{\alpha / 2} ( x ) ) / 2 \in \mathds{R}_{\geq 0}$.
Then, the non-conformity function can be rewritten as $R ( x , y ) = \max \{ M ( x ) - M_{\Delta} ( x ) - y , y - M ( x ) - M_{\Delta} ( x ) \}$ and the space of non-conformity scores is $\mathcal{R} = \mathds{R}$.

The multiplicative factor in this setting is given by,
\begin{equation*}
\begin{split}
    & \mf \!=\! \int_{\mathds{R}} \pdf{R ( X_{n + 1} , y )}{r} d y \!=\! \int_{\mathds{R}} \frac{d}{d r} \cdf{R ( X_{n + 1} , y )}{r} d y \!=\! \frac{d}{d r} \int_{\mathds{R}} \cdf{R ( X_{n + 1} , y )}{r} d y \!=\! \frac{d}{d r} \int_{\mathds{R}} \pr{R ( X_{n + 1} , y ) \leq r} d y \\
    & \!=\! \frac{d}{d r} \int_{\mathds{R}} \pr{\max \{ M ( X_{n + 1} ) - M_{\Delta} ( X_{n + 1} ) - y , y -  M ( X_{n + 1} ) - M_{\Delta} ( X_{n + 1} ) \} \leq r} d y \\
    & \!=\! \frac{d}{d r} \int_{\mathds{R}} \pr{y \in \left[  M ( X_{n + 1} ) - M_{\Delta} ( X_{n + 1} ) - r ,  M ( X_{n + 1} ) + M_{\Delta} ( X_{n + 1} ) + r \right]} d y \\
    & \!=\! \frac{d}{d r} \int_{\mathds{R}} \int_{\mathds{R}_{\geq 0}} \int_{\mathds{R}} \pr{y \!\in\! \left[  M ( X_{n + 1} ) - M_{\Delta} ( X_{n + 1} ) - r ,  M ( X_{n + 1} ) + M_{\Delta} ( X_{n + 1} ) + r \right] \middle\vert M ( X_{n + 1} ) \!=\! c , M_{\Delta} ( X_{n + 1} ) \!=\! \delta} \\
    & \quad \quad \quad \quad \quad \quad \quad \ \pdf{M ( X_{n + 1} ) , M_{\Delta} ( X_{n + 1} )}{c , \delta} d c \, d \delta \, d y \\
    & \!=\! \frac{d}{d r} \int_{\mathds{R}} \int_{\mathds{R}_{\geq 0}} \int_{\mathds{R}} \ind{y \in \left[  c - \delta - r ,  c + \delta + r \right]} \pdf{M ( X_{n + 1} ) , M_{\Delta} ( X_{n + 1} )}{c , \delta} d c \, d \delta \, d y \\
    & \!=\! \frac{d}{d r} \int_{\mathds{R}_{\geq 0}} \int_{\mathds{R}} \pdf{M ( X_{n + 1} ) , M_{\Delta} ( X_{n + 1} )}{c , \delta} \int_{\mathds{R}} \ind{y \in \left[  c - \delta - r ,  c + \delta + r \right]} d y \, d c \, d \delta \\
    & \!=\! \frac{d}{d r} \int_{\mathds{R}_{\geq 0}} \int_{\mathds{R}} \pdf{M ( X_{n + 1} ) , M_{\Delta} ( X_{n + 1} )}{c , \delta} 2 ( \delta + r ) \ind{\delta + r \geq 0} d c \, d \delta \\
    & \!=\! \frac{d}{d r} \int_{\mathds{R}_{\geq 0}} 2 ( \delta + r ) \ind{\delta + r \geq 0} \int_{\mathds{R}} \pdf{M ( X_{n + 1} ) , M_{\Delta} ( X_{n + 1} )}{c , \delta} d c \, d \delta \\
    & \!=\! \frac{d}{d r} \int_{\mathds{R}_{\geq 0}} 2 ( \delta + r ) \ind{\delta + r \geq 0} \pdf{M_{\Delta} ( X_{n + 1} )}{\delta} d \delta \!=\! \frac{d}{d r} \int_{\max \{ 0 , - r \}}^{\infty} 2 ( \delta + r ) \pdf{M_{\Delta} ( X_{n + 1} )}{\delta} d \delta \\
    & \!=\! 2 \int_{\max \{ 0 , - r \}}^{\infty} \frac{d}{d r} ( \delta + r ) \pdf{M_{\Delta} ( X_{n + 1} )}{\delta} d \delta \!=\! 2 \int_{\max \{ 0 , - r \}}^{\infty} \pdf{M_{\Delta} ( X_{n + 1} )}{\delta} d \delta \!=\! 2 \pr{M_{\Delta} ( X_{n + 1} ) \geq \max \{ 0 , - r \}} \\
    & \!=\! 2 \left( 1 - \pr{M_{\Delta} ( X_{n + 1} ) < \max \{ 0 , - r \}} \right) \!\overset{(i)}{=}\! 2 \left( 1 - \pr{M_{\Delta} ( X_{n + 1} ) \leq \max \{ 0 , - r \}} \right) \\
    & \!=\! 2 \left( 1 - \cdf{M_{\Delta} ( X_{n + 1} )}{\max \{ 0 , - r \}} \right) \!=\! \begin{cases} 2 \left( 1 - \cdf{M_{\Delta} ( X_{n + 1} )}{0} \right) , & r \geq 0 \\ 2 \left( 1 - \cdf{M_{\Delta} ( X_{n + 1} )}{- r} \right) , & r < 0 \end{cases} \!=\! \begin{cases} 2 , & r \geq 0 \\ 2 \left( 1 - \cdf{M_{\Delta} ( X_{n + 1} )}{- r} \right) , & r < 0 \end{cases},
\end{split}
\end{equation*}
where $(i)$ follows from $M_{\Delta} ( X_{n + 1} )$ being continuous.
The multiplicative factor is again dependent on the data distribution and the machine learning model, and is therefore intractable without making further assumptions.

\paragraph{Adaptive Prediction Sets (APS) and Regularized Adaptive Prediction Sets (RAPS)}

\citet{romano2020classification,angelopoulos2021uncertainty} propose non-conformity functions for classification.
\citet{romano2020classification} propose adaptive prediction sets (APS) that sum the predicted label probabilities in descending order until the label assigned to the data point is included; the corresponding non-conformity function is given by $R ( x , y ) = U M_{y} ( x ) + \sum_{y^{\prime} \in \mathcal{Y}} \ind{M_{y^{\prime}} ( x ) > M_{y} ( x )} M_{y^{\prime}} ( x )$, where $M_{y} (x)$ is the predicted probability for label $y \in \mathcal{Y}$ for an input feature $x \in \mathcal{X}$ and $U \sim \mathcal{U} ( 0 , 1 )$ is a uniform random variable over $[ 0 , 1 ]$.
Additionally, \citet{angelopoulos2021uncertainty} propose regularized adaptive prediction sets (RAPS) that further add a regularization term to penalize the number of labels included in the prediction set.
Both these non-conformity functions construct small prediction sets, but their associated multiplicative factors are intractable without making further assumptions.

\section{EXPERIMENTS ON UCI DATASETS}
\label{appendix:experiments-uci}

\begin{table}[t]
    \caption{
        \textbf{Dataset statistics summaries.}
        We summarize the statistics of the UCI datasets used in our experiments.
        This includes the number of data points and features for each dataset.
        For regression datasets, we also include the range of label values; for classification, we also include the number of classes/labels.
    }
    \label{table:uci-dataset}
    \vspace{-10pt}
    \begin{center}
    \begin{scriptsize}
    \begin{tabular}{lC{1.6cm}C{1.5cm}C{1.6cm}lC{1.6cm}C{1.5cm}C{1.2cm}}
        \toprule
        \multicolumn{4}{c}{\textbf{Regression ($\pmb{\mathcal{Y} = \mathds{R}}$)}} & \multicolumn{4}{c}{\textbf{Classification (discrete $\pmb{\mathcal{Y}}$)}} \\
        \cmidrule(lr){1-4} \cmidrule(lr){5-8}
        \textbf{Dataset} & \textbf{Number of data points} & \textbf{Number of features} & \textbf{Range of labels} & \textbf{Dataset} & \textbf{Number of data points} & \textbf{Number of features} & \textbf{Number of labels} \\
        \midrule
        Abalone & 4177 & 10 & [-2.79, 5.94] & APSFailure & 120000 & 341 & 2 \\
        AirFoil & 1503 & 5 & [-3.11, 2.34] & Adult & 48842 & 108 & 2 \\
        AirQuality & 8991 & 14 & [-1.35, 7.20] & Avila & 20867 & 10 & 12 \\
        BlogFeedback & 60021 & 280 & [-0.72, 3.20] & BankMarketing & 41188 & 63 & 2 \\
        CTSlices & 53500 & 384 & [-2.00, 2.25] & CardDefault & 30000 & 23 & 2 \\
        FacebookComments & 209074 & 53 & [-0.21, 70.12] & Landsat & 6435 & 36 & 6 \\
        OnlineNews & 39644 & 58 & [-8.03, 6.63] & LetterRecognition & 20000 & 16 & 26 \\
        PowerPlant & 9568 & 4 & [-2.00, 2.43] & MagicGamma & 19020 & 10 & 2 \\
        Superconductivity & 21263 & 81 & [-1.00, 4.39] & SensorLessDrive & 58509 & 48 & 11 \\
        WhiteWineQuality & 4898 & 11 & [-3.32, 3.57] & Shuttle & 58000 & 9 & 7 \\
        \bottomrule
    \end{tabular}
    \end{scriptsize}
    \end{center}
    \vspace{-10pt}
\end{table}

We illustrate the efficacy of our results experimentally by applying our estimation procedures derived in~\cref{section:size-estimation} on real-world datasets from the UCI database~\citep{kelly2023uci}.\footnote{We use the python package~\href{https://github.com/isacarnekvist/ucimlr}{https://github.com/isacarnekvist/ucimlr} to access the datasets.}
We summarize the dataset statistics in~\cref{table:uci-dataset}.

\subsection{Prediction Errors}
\label{appendix:experiments-uci-error}

\begin{table}[t]
    \caption{
        \textbf{Prediction error frequencies.}
        We report the empirically achieved prediction error frequencies for the split conformal prediction framework using different non-conformity functions and UCI datasets (with $\alpha = 0. 1$).
    }
    \label{table:uci-error}
    \vspace{-10pt}
    \begin{center}
    \begin{scriptsize}
    \begin{tabular}{lC{1.5cm}C{1.5cm}clC{1.5cm}C{1.5cm}C{1.5cm}}
        \toprule
        \multicolumn{3}{c}{\textbf{Regression ($\pmb{\mathcal{Y} = \mathds{R}}$)}} & & \multicolumn{4}{c}{\textbf{Classification (discrete $\pmb{\mathcal{Y}}$)}} \\
        \cmidrule(lr){1-3} \cmidrule(lr){5-8}
        & \multicolumn{2}{c}{\textbf{Prediction error frequency}} & & & \multicolumn{3}{c}{\textbf{Prediction error frequency}} \\
        \cmidrule(lr){2-3} \cmidrule(lr){6-8}
        \textbf{Dataset} & \textbf{$\pmb{l_{1}}$} & \textbf{CQR} & & \textbf{Dataset} & \textbf{0-1} & \textbf{LAC} & \textbf{APS} \\
        \midrule
        Abalone & 0.0987 & 0.0846 & & APSFailure (2) & 0.0055 & 0.0742 & 0.0999 \\
        AirFoil & 0.0987 & 0.0983 & & Adult (2) & 0.0000 & 0.0982 & 0.1001 \\
        AirQuality & 0.0987 & 0.0201 & & Avila (12) & 0.0489 & 0.0973 & 0.0998 \\
        BlogFeedback & 0.1001 & 0.0514 & & BankMarketing (2) & 0.0904 & 0.0982 & 0.0999 \\
        CTSlices & 0.1000 & 0.0999 & & CardDefault (2) & 0.0000 & 0.0984 & 0.0999 \\
        FacebookComments & 0.1000 & 0.0449 & & Landsat (6) & 0.0255 & 0.0982 & 0.1001 \\
        OnlineNews & 0.1000 & 0.0999 & & LetterRecognition (26) & 0.0781 & 0.0972 & 0.0999 \\
        PowerPlant & 0.0997 & 0.0998 & & MagicGamma (2) & 0.0000 & 0.0983 & 0.0999 \\
        Superconductivity & 0.0999 & 0.0996 & & SensorLessDrive (11) & 0.0036 & 0.0938 & 0.0999 \\
        WhiteWineQuality & 0.0977 & 0.0387 & & Shuttle (7) & 0.0007 & 0.0144 & 0.0999 \\
        \bottomrule
    \end{tabular}
    \end{scriptsize}
    \end{center}
    \vspace{-10pt}
\end{table}

We include the empirically achieved prediction error frequencies for our implementation of the split conformal prediction framework using different non-conformity functions.
This facilitates the evaluation of our implementation in satisfying the requirement in~\cref{equation:conformal-error}.
We use the same setup as the one highlighted in~\cref{section:experiments-setup}.

With the significance level $\alpha$ set to 0.1, the results are illustrated in~\cref{table:uci-error}.
We observe that the error frequencies are either close to or less than the desired bound of $\alpha = 0.1$ for every non-conformity function and dataset.

\subsection{Interval Estimate Errors}
\label{appendix:experiments-uci-interval-error}

Our estimation procedures in~\cref{section:size-estimation} provide point and interval estimates for the expected prediction set size.
The latter are high-probability bounds, which in fact are valid confidence intervals when the multiplicative factor is known (cf.~\cref{corollary:size-confidence-interval}).
We want to validate our results experimentally; however, it is impossible to do so as the true expected set size is unknown in practice.
As a proxy, we use the mean Monte Carlo average instead and test its inclusion in our individual interval bounds.
With that, we expand on our experimental results in~\cref{section:experiments-marginal}.
Note that our interval estimates are obtained from~\cref{subsection:size-estimation-known-multiplicative-factor} (with valid confidence intervals) for the $l_{1}$ and 0-1 loss non-conformity functions, and from~\cref{subsection:size-estimation-unknown-multiplicative-factor} for the other non-conformity functions.

\cref{table:uci-interval-error} illustrates the frequency of error of our individual estimated intervals (with $\gamma$ set to 0.1) bounding the mean Monte Carlo average; we would expect these values to be below $\gamma = 0.1$.
When the intervals are valid confidence intervals (under the $l_{1}$ and 0-1 loss non-conformity functions), the error frequencies are always below 0.1, corroborating our result in~\cref{corollary:size-confidence-interval}.
When the intervals are not necessarily valid confidence intervals, they still achieve errors lower than 0.1 on 23/30 instances.
In the 7 remaining instances, our point and interval estimates are close to the mean Monte Carlo average, but the standard deviation in the Monte Carlo average estimate itself is high.
Note that this is a proxy to the true interval error, which cannot be computed in practice.

\begin{table}[t]
    \caption{
        \textbf{Interval error frequencies.}
        We report the error frequencies of our individual interval estimates (with $\gamma = 0.1$) bounding the mean Monte Carlo average for different non-conformity functions and UCI datasets.
    }
    \label{table:uci-interval-error}
    \vspace{-10pt}
    \begin{center}
    \begin{scriptsize}
    \begin{tabular}{lC{1.5cm}C{1.5cm}clC{1.5cm}C{1.5cm}C{1.5cm}}
        \toprule
        \multicolumn{3}{c}{\textbf{Regression ($\pmb{\mathcal{Y} = \mathds{R}}$)}} & & \multicolumn{4}{c}{\textbf{Classification (discrete $\pmb{\mathcal{Y}}$)}} \\
        \cmidrule(lr){1-3} \cmidrule(lr){5-8}
        & \multicolumn{2}{c}{\textbf{Interval error frequency}} & & & \multicolumn{3}{c}{\textbf{Interval error frequency}} \\
        \cmidrule(lr){2-3} \cmidrule(lr){6-8}
        \textbf{Dataset} & \textbf{$\pmb{l_{1}}$} & \textbf{CQR} & & \textbf{Dataset} & \textbf{0-1} & \textbf{LAC} & \textbf{APS} \\
        \midrule
        Abalone & 0.0000 & 0.0000 & & APSFailure (2) & 0.0000 & 0.4420 & 0.0040 \\
        AirFoil & 0.0000 & 0.0000 & & Adult (2) & 0.0000 & 0.0000 & 0.0000 \\
        AirQuality & 0.0020 & 1.0000 & & Avila (12) & 0.0000 & 0.0000 & 0.0080 \\
        BlogFeedback & 0.0000 & 1.0000 & & BankMarketing (2) & 0.0370 & 0.0000 & 0.0020 \\
        CTSlices & 0.0050 & 0.8820 & & CardDefault (2) & 0.0000 & 0.0000 & 0.0000 \\
        FacebookComments & 0.0000 & 1.0000 & & Landsat (6) & 0.0000 & 0.0000 & 0.0020 \\
        OnlineNews & 0.0000 & 0.0000 & & LetterRecognition (26) & 0.0000 & 0.0000 & 0.0000 \\
        PowerPlant & 0.0000 & 0.0000 & & MagicGamma (2) & 0.0000 & 0.0000 & 0.0010 \\
        Superconductivity & 0.0000 & 0.0330 & & SensorLessDrive (11) & 0.0000 & 0.0000 & 0.0000 \\
        WhiteWineQuality & 0.0000 & 0.8920 & & Shuttle (7) & 0.0000 & 1.0000 & 0.0050 \\
        \bottomrule
    \end{tabular}
    \end{scriptsize}
    \end{center}
    \vspace{-10pt}
\end{table}

\subsection{High-Dimensional Regression}
\label{appendix:experiments-uci-high-dimensional-regression}

\begin{table}[t]
    \caption{
        \textbf{Marginal expected prediction set sizes (high-dimensional regression).}
        We illustrate the marginal expected prediction set sizes.
        The estimates are obtained via Monte Carlo averaging, our point estimates, and our interval estimates (lower-upper bounds with $\gamma = 0.1$).
        We also compute the absolute errors between our individual point estimates and the mean Monte Carlo average, and the frequencies of error of our individual interval estimates bounding the mean Monte Carlo average.
        We report the means and standard deviations.
    }
    \label{table:uci-high-dimensional-regression-size}
    \vspace{-10pt}
    \begin{center}
    \begin{scriptsize}
    \begin{tabular}{llC{1.8cm}C{1.8cm}C{1.8cm}C{1.8cm}C{1.8cm}C{1.9cm}}
        \toprule
        & & \multicolumn{4}{c}{\textbf{Marginal expected prediction set size}} & \\
        \cmidrule(lr){3-6}
        & \textbf{Dataset} & \textbf{Our interval lower bound} & \textbf{Monte Carlo average} & \textbf{Our point estimate} & \textbf{Our interval upper bound} & \textbf{Absolute error} & \textbf{Interval error frequency} \\
        \midrule
        $\pmb{l_{1}}$ & \multirow{2}{*}{Parkinson} & 1.33\textsubscript{0.20} & 1.85\textsubscript{0.29} & 1.85\textsubscript{0.28} & 2.71\textsubscript{0.43} & 0.24\textsubscript{0.14} & 0.0010 \\
        $\pmb{l_{2}}$ & & 1.07\textsubscript{0.16} & 1.48\textsubscript{0.23} & 1.48\textsubscript{0.22} & 2.16\textsubscript{0.34} & 0.19\textsubscript{0.11} & 0.0000 \\
        \bottomrule
    \end{tabular}
    \end{scriptsize}
    \end{center}
    \vspace{-10pt}
\end{table}

Here we consider the Parkinson dataset from the UCI database~\citep{kelly2023uci}, a 2-dimensional regression dataset with 5875 data points and 19-dimensional features.
We use the $l_{1}$ and the $l_{2}$ loss non-conformity functions with multiplicative factors $\mf = 4 r$ and $\mf = \pi$ respectively (cf.~\cref{appendix:multiplicative-factor-lp-regression}).
As the multiplicative factors are known, we compute our empirical estimates from~\cref{subsection:size-estimation-known-multiplicative-factor}.
We use the same setup as in~\cref{section:experiments-setup}.

\cref{table:uci-high-dimensional-regression-size} illustrates the Monte Carlo average and our estimates for the marginal expected prediction set size on the Parkinson dataset.
We observe trends similar to those in~\cref{table:uci-size}; the means of our point estimates are close to that of the Monte Carlo average, with the standard deviations being comparable despite using $3 \times$ fewer data points.
This is also reflected in the low absolute error between our individual point estimates and the mean Monte Carlo average.
Additionally, our interval estimates provide lower-upper bounds on the expected set size.
Similar to~\cref{table:uci-interval-error}, the error frequencies of our individual estimated intervals bounding the mean Monte Carlo average are below $\gamma = 0.1$.
These results corroborate the efficacy of our estimates on high-dimensional regression problems.

\subsection{Insights Ablation}
\label{appendix:experiments-uci-insights}

\begin{table}[t]
    \caption{
        \textbf{Marginal expected prediction set sizes (insights ablation).}
        We illustrate changes in the marginal expected prediction set sizes (the Monte Carlo averages) using different non-conformity functions and UCI datasets.
        The first column corresponds to no change from the setup in~\cref{section:experiments-setup}.
        The second corresponds to an increase in the amount of training data.
        The third corresponds to a decrease in the significance level.
        The fourth corresponds to an increase in the amount of calibration data.
        We report the means and standard deviations.
    }
    \label{table:uci-insights-size}
    \vspace{-10pt}
    \begin{center}
    \begin{scriptsize}
    \begin{tabular}{lllC{2.4cm}C{2.4cm}C{2.4cm}C{2.4cm}}
        \toprule
        & & & \multicolumn{4}{c}{\textbf{Marginal expected prediction set size (Monte Carlo average)}} \\
        \cmidrule(lr){4-7}
        & & \textbf{Dataset} & \textbf{No change} & \textbf{Increase in training data} & \textbf{Decrease in significance level} & \textbf{Increase in calibration data} \\
        \midrule
        \multirow{10}{*}{\rotatebox[origin=c]{90}{\textbf{Regression ($\pmb{\mathcal{Y} = \mathds{R}}$)}}} & \multirow{5}{*}{\rotatebox[origin=c]{90}{$\pmb{l_{1}}$}} & Abalone & 2.19\textsubscript{0.09} & 2.14\textsubscript{0.09} & 4.86\textsubscript{0.36} & 2.18\textsubscript{0.06} \\
        & & AirFoil & 1.39\textsubscript{0.10} & 1.04\textsubscript{0.08} & 2.80\textsubscript{0.39} & 1.38\textsubscript{0.08} \\
        & & AirQuality & 0.02\textsubscript{0.00} & 0.01\textsubscript{0.00} & 0.07\textsubscript{0.02} & 0.02\textsubscript{0.00} \\
        & & PowerPlant & 0.70\textsubscript{0.02} & 0.64\textsubscript{0.01} & 1.32\textsubscript{0.06} & 0.70\textsubscript{0.01} \\
        & & WhiteWineQuality & 2.58\textsubscript{0.08} & 2.43\textsubscript{0.08} & 4.81\textsubscript{0.27} & 2.57\textsubscript{0.05} \\
        \cmidrule(lr){2-7}
        & \multirow{5}{*}{\rotatebox[origin=c]{90}{\textbf{CQR}}} & Abalone & 2.17\textsubscript{0.98} & 1.96\textsubscript{0.97} & 4.09\textsubscript{1.15} & 2.19\textsubscript{0.97} \\
        & & AirFoil & 1.58\textsubscript{0.64} & 1.22\textsubscript{0.56} & 3.01\textsubscript{0.81} & 1.57\textsubscript{0.63} \\
        & & AirQuality & 0.02\textsubscript{0.11} & 0.01\textsubscript{0.06} & 0.05\textsubscript{0.19} & 0.02\textsubscript{0.11} \\
        & & PowerPlant & 0.73\textsubscript{0.26} & 0.68\textsubscript{0.25} & 1.27\textsubscript{0.32} & 0.73\textsubscript{0.26} \\
        & & WhiteWineQuality & 2.24\textsubscript{0.89} & 2.11\textsubscript{0.93} & 5.08\textsubscript{1.18} & 2.24\textsubscript{0.89} \\
        \midrule
        \multirow{15}{*}{\rotatebox[origin=c]{90}{\textbf{Classification (discrete $\pmb{\mathcal{Y}}$)}}} & \multirow{5}{*}{\rotatebox[origin=c]{90}{\textbf{0-1}}} & Avila (12) & 1.00\textsubscript{0.00} & 1.00\textsubscript{0.00} & 12.00\textsubscript{0.00} & 1.00\textsubscript{0.00} \\
        & & CardDefault (2) & 2.00\textsubscript{0.00} & 2.00\textsubscript{0.00} & 2.00\textsubscript{0.00} & 2.00\textsubscript{0.00} \\
        & & Landsat (6) & 4.79\textsubscript{2.14} & 1.35\textsubscript{1.28} & 6.00\textsubscript{0.00} & 5.25\textsubscript{1.78} \\
        & & LetterRecognition (26) & 1.00\textsubscript{0.00} & 1.00\textsubscript{0.00} & 26.00\textsubscript{0.00} & 1.00\textsubscript{0.00} \\
        & & MagicGamma (2) & 2.00\textsubscript{0.00} & 2.00\textsubscript{0.00} & 2.00\textsubscript{0.00} & 2.00\textsubscript{0.00} \\
        \cmidrule(lr){2-7}
        & \multirow{5}{*}{\rotatebox[origin=c]{90}{\textbf{LAC}}} & Avila (12) & 0.93\textsubscript{0.26} & 0.90\textsubscript{0.29} & 1.26\textsubscript{0.48} & 0.93\textsubscript{0.26} \\
        & & CardDefault (2) & 1.25\textsubscript{0.44} & 1.25\textsubscript{0.43} & 1.87\textsubscript{0.34} & 1.25\textsubscript{0.44} \\
        & & Landsat (6) & 1.02\textsubscript{0.25} & 0.99\textsubscript{0.22} & 1.69\textsubscript{0.93} & 1.02\textsubscript{0.25} \\
        & & LetterRecognition (26) & 0.97\textsubscript{0.32} & 0.93\textsubscript{0.29} & 2.29\textsubscript{1.65} & 0.97\textsubscript{0.32} \\
        & & MagicGamma (2) & 1.07\textsubscript{0.26} & 1.06\textsubscript{0.24} & 1.63\textsubscript{0.48} & 1.07\textsubscript{0.26} \\
        \cmidrule(lr){2-7}
        & \multirow{5}{*}{\rotatebox[origin=c]{90}{\textbf{APS}}} & Avila (12) & 1.22\textsubscript{0.69} & 1.09\textsubscript{0.61} & 2.33\textsubscript{1.37} & 1.22\textsubscript{0.69} \\
        & & CardDefault (2) & 1.36\textsubscript{0.50} & 1.35\textsubscript{0.50} & 1.91\textsubscript{0.29} & 1.36\textsubscript{0.50} \\
        & & Landsat (6) & 1.32\textsubscript{0.78} & 1.27\textsubscript{0.74} & 2.27\textsubscript{1.39} & 1.31\textsubscript{0.78} \\
        & & LetterRecognition (26) & 2.49\textsubscript{2.63} & 2.20\textsubscript{2.42} & 6.80\textsubscript{6.11} & 2.49\textsubscript{2.63} \\
        & & MagicGamma (2) & 1.21\textsubscript{0.49} & 1.19\textsubscript{0.49} & 1.74\textsubscript{0.44} & 1.21\textsubscript{0.49} \\
        \bottomrule
    \end{tabular}
    \end{scriptsize}
    \end{center}
    \vspace{-10pt}
\end{table}

We analyzed the dependence of the expected size of prediction sets on various user-specified parameters in~\cref{subsection:size-quantification-insights}.
Here we empirically validate our analysis by providing experimental results on such parameter changes.

We use the same setup as the one highlighted in~\cref{section:experiments-setup}, and further add 3 settings, changing a user-specified parameter one at a time.
These settings are:~(i) an increase in the amount of training data from 25\% to 50\% of the dataset, hence learning a machine learning model that generalizes better (which is further used to implement the non-conformity function), (ii) a decrease in the significance level $\alpha$ from 0.1 to 0.01, allowing for fewer errors in the conformal system, and (iii) an increase in the amount of calibration data $n$ from 25\% to 50\% of the dataset.

\cref{table:uci-insights-size} illustrates the change in the marginal expected prediction set size (the Monte Carlo average) under varying user-specified parameters.
We observe that:~(i) an increase in the amount of training data decreases the expected prediction set size, (ii) a decrease in the significance level increases the expected prediction set size, and (iii) an increase in the calibration data does not affect the expected prediction set size by much; the only exception is when using the 0-1 loss non-conformity function on Landsat, but the standard deviations of the estimates are high under this setting.
These experimental results empirically validate our analysis in~\cref{subsection:size-quantification-insights}.

\subsection{Conditional Expected Prediction Set Size}
\label{appendix:experiments-uci-conditional}

\begin{figure}[t]
    \begin{center}
    \begin{subfigure}[b]{\textwidth}
        \includegraphics[width=\textwidth]{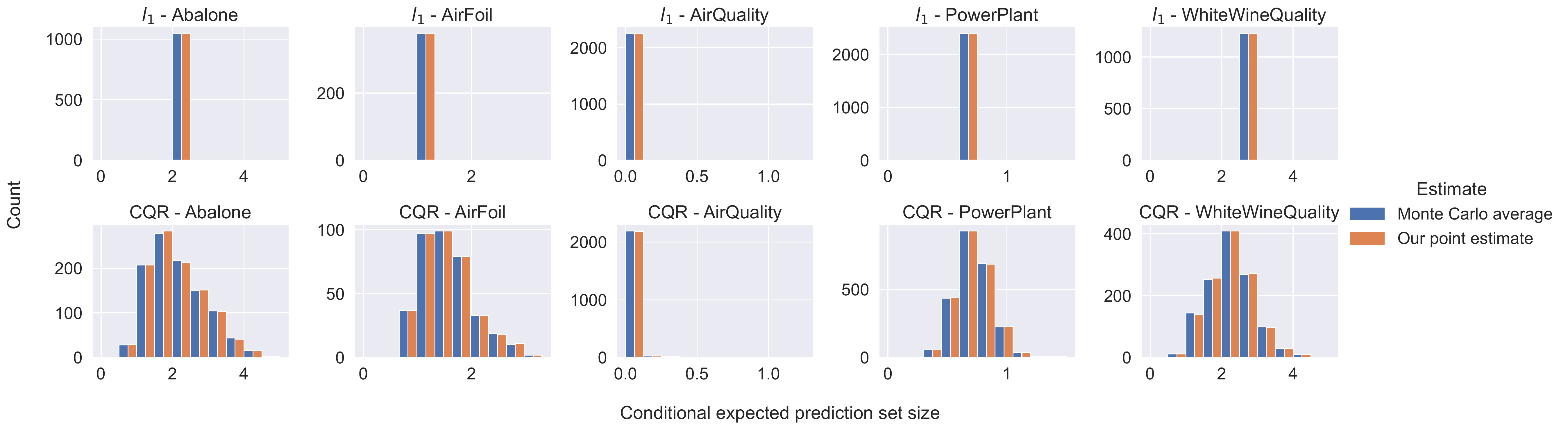}
    \end{subfigure}
    \begin{subfigure}[b]{\textwidth}
        \includegraphics[width=\textwidth]{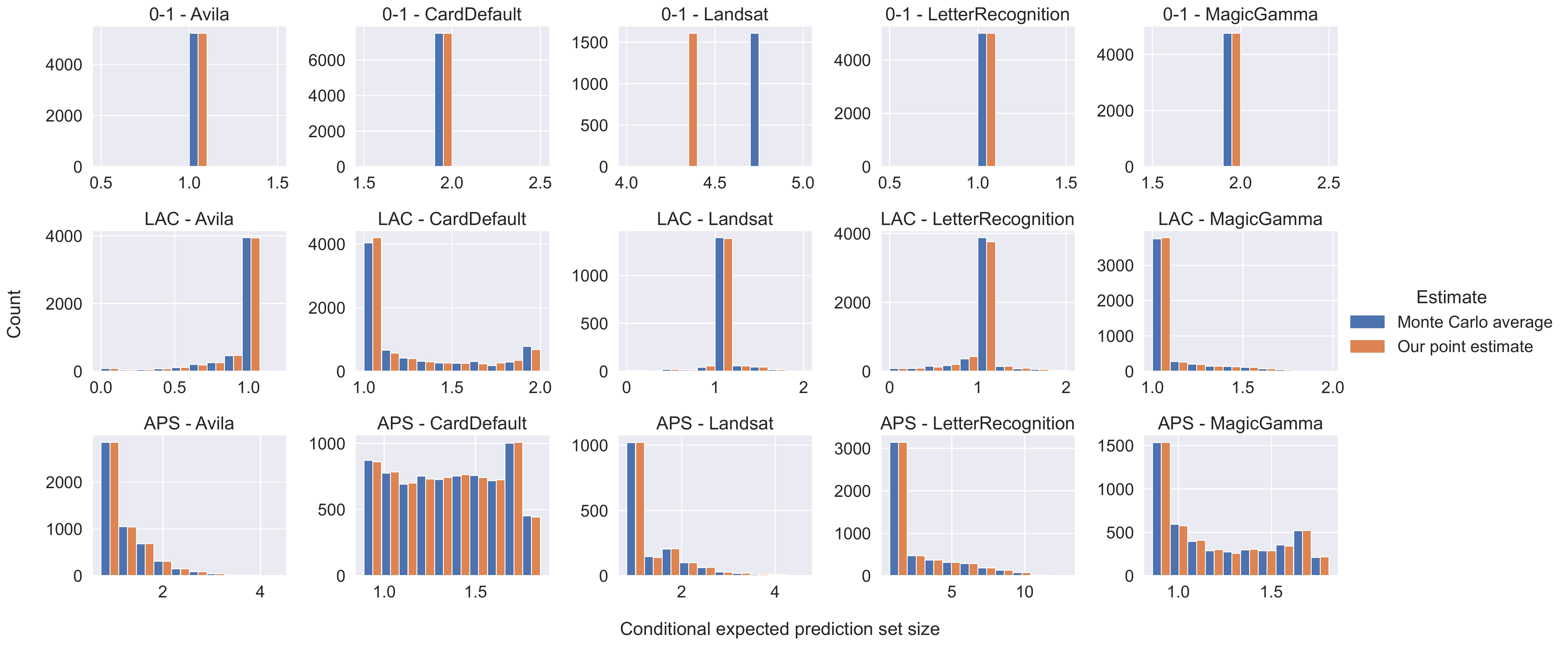}
    \end{subfigure}
    \end{center}
    \vspace{-10pt}
    \caption{
        \textbf{Expected prediction set sizes conditioned on the test datum feature.}
        We illustrate the expected sizes of split conformal prediction sets conditioned on varying test datum features using different non-conformity scores (rows) and UCI datasets (columns).
        The estimates are obtained via Monte Carlo averaging and our point estimates (refer to the legend for the color scheme); they are depicted as a histogram with side-by-side bars.
    }
    \label{figure:uci-test-conditional-size-all}
    \vspace{-10pt}
\end{figure}

Here we analyze the expected size of prediction sets conditioned on the test inputs.
This is an extension of our experimental results in~\cref{section:experiments-conditional}, where we considered the non-conformity functions CQR~\citep{romano2019conformalized} and APS~\citep{romano2020classification}; here we extend our analysis to other non-conformity functions as well.

\cref{figure:uci-test-conditional-size-all} depicts histograms of the Monte Carlo average and our point estimates for the expected set sizes conditioned on varying test inputs.
Similar to our observations from~\cref{figure:uci-test-conditional-size}, the histograms in~\cref{figure:uci-test-conditional-size-all} look identical for the two estimates, despite our point estimate not having seen the calibration data.
The only exception is when using the 0-1 loss non-conformity function on Landsat.
These results further corroborate the efficacy of our estimates in approximating the expected size of prediction sets conditioned on the test datum feature.

\subsection{Other Estimates}
\label{appendix:experiments-uci-other-estimates}

\begin{table}[!t]
    \caption{
        \textbf{Marginal expected prediction set sizes (point estimates).}
        We compare different point estimates for the marginal expected prediction set size with respect to regular Monte Carlo averaging.
        They are obtained via our point estimates and the same-data Monte Carlo average.
        We also compute the absolute errors between the individual estimates and the mean Monte Carlo average.
        We report the means and standard deviations.
    }
    \label{table:uci-other-point-estimates}
    \vspace{-10pt}
    \begin{center}
    \begin{scriptsize}
    \begin{tabular}{lllC{1.8cm}C{1.8cm}C{1.8cm}C{1.8cm}C{1.8cm}}
        \toprule
        & & & \multicolumn{3}{c}{\textbf{Marginal expected prediction set size}} & \multicolumn{2}{c}{\textbf{Absolute error}} \\
        \cmidrule(lr){4-6} \cmidrule(lr){7-8}
        & & \textbf{Dataset} & \textbf{Monte Carlo average} & \textbf{Our point estimate} & \textbf{Same-data Monte Carlo average} & \textbf{Our point estimate} & \textbf{Same-data Monte Carlo average} \\
        \midrule
        \multirow{10}{*}{\rotatebox[origin=c]{90}{\textbf{Regression ($\pmb{\mathcal{Y} = \mathds{R}}$)}}} & \multirow{5}{*}{\rotatebox[origin=c]{90}{$\pmb{l_{1}}$}} & Abalone & 2.19\textsubscript{0.09} & 2.19\textsubscript{0.09} & 2.20\textsubscript{0.13} & 0.07\textsubscript{0.05} & 0.11\textsubscript{0.08} \\
        & & AirFoil & 1.39\textsubscript{0.10} & 1.39\textsubscript{0.09} & 1.41\textsubscript{0.13} & 0.08\textsubscript{0.05} & 0.10\textsubscript{0.08} \\
        & & AirQuality & 0.02\textsubscript{0.00} & 0.02\textsubscript{0.00} &  0.02\textsubscript{0.00} & 0.00\textsubscript{0.00} & 0.00\textsubscript{0.00} \\
        & & PowerPlant & 0.70\textsubscript{0.02} & 0.70\textsubscript{0.02} & 0.70\textsubscript{0.02} & 0.01\textsubscript{0.01} & 0.02\textsubscript{0.01} \\
        & & WhiteWineQuality & 2.58\textsubscript{0.08} & 2.58\textsubscript{0.07} & 2.58\textsubscript{0.11} & 0.06\textsubscript{0.05} & 0.09\textsubscript{0.07} \\
        \cmidrule(lr){2-8}
        & \multirow{5}{*}{\rotatebox[origin=c]{90}{\textbf{CQR}}} & Abalone & 2.17\textsubscript{0.98} & 2.16\textsubscript{0.17} & 2.18\textsubscript{0.28} & 0.15\textsubscript{0.09} & 0.27\textsubscript{0.09} \\
        & & AirFoil & 1.58\textsubscript{0.64} & 1.58\textsubscript{0.07} & 1.60\textsubscript{0.11} & 0.05\textsubscript{0.04} & 0.09\textsubscript{0.07} \\
        & & AirQuality & 0.02\textsubscript{0.11} & 0.02\textsubscript{0.00} & 0.02\textsubscript{0.00} & 0.00\textsubscript{0.00} & 0.00\textsubscript{0.00} \\
        & & PowerPlant & 0.73\textsubscript{0.26} & 0.73\textsubscript{0.01} & 0.74\textsubscript{0.02} & 0.01\textsubscript{0.01} & 0.01\textsubscript{0.01} \\
        & & WhiteWineQuality & 2.24\textsubscript{0.89} & 2.24\textsubscript{0.10} & 2.24\textsubscript{0.10} & 0.07\textsubscript{0.06} & 0.08\textsubscript{0.06} \\
        \midrule
        \multirow{15}{*}{\rotatebox[origin=c]{90}{\textbf{Classification (discrete $\pmb{\mathcal{Y}}$)}}} & \multirow{5}{*}{\rotatebox[origin=c]{90}{\textbf{0-1}}} & Avila (12) & 1.00\textsubscript{0.00} & 1.00\textsubscript{0.00} & 1.00\textsubscript{0.00} & 0.00\textsubscript{0.00} & 0.00\textsubscript{0.00} \\
        & & CardDefault (2) & 2.00\textsubscript{0.00} & 2.00\textsubscript{0.00} & 2.00\textsubscript{0.00} & 0.00\textsubscript{0.00} & 0.00\textsubscript{0.00} \\
        & & Landsat (6) & 4.79\textsubscript{2.14} & 4.47\textsubscript{1.31} & 4.61\textsubscript{2.24} & 1.07\textsubscript{0.82} & 1.93\textsubscript{1.16} \\
        & & LetterRecognition (26) & 1.00\textsubscript{0.00} & 1.00\textsubscript{0.01} & 1.00\textsubscript{0.00} & 0.00\textsubscript{0.01} & 0.00\textsubscript{0.00} \\
        & & MagicGamma (2) & 2.00\textsubscript{0.00} & 2.00\textsubscript{0.00} & 2.00\textsubscript{0.00} & 0.00\textsubscript{0.00} & 0.00\textsubscript{0.00} \\
        \cmidrule(lr){2-8}
        & \multirow{5}{*}{\rotatebox[origin=c]{90}{\textbf{LAC}}} & Avila (12) & 0.93\textsubscript{0.26} & 0.93\textsubscript{0.01} & 0.93\textsubscript{0.01} & 0.00\textsubscript{0.00} & 0.01\textsubscript{0.01} \\
        & & CardDefault (2) & 1.25\textsubscript{0.44} & 1.25\textsubscript{0.01} & 1.25\textsubscript{0.02} & 0.01\textsubscript{0.01} & 0.02\textsubscript{0.01} \\
        & & Landsat (6) & 1.02\textsubscript{0.25} & 1.02\textsubscript{0.02} & 1.02\textsubscript{0.03} & 0.01\textsubscript{0.01} & 0.02\textsubscript{0.02} \\
        & & LetterRecognition (26) & 0.97\textsubscript{0.32} & 0.97\textsubscript{0.01} & 0.97\textsubscript{0.02} & 0.01\textsubscript{0.00} & 0.01\textsubscript{0.01} \\
        & & MagicGamma (2) & 1.07\textsubscript{0.26} & 1.07\textsubscript{0.01} & 1.07\textsubscript{0.02} & 0.01\textsubscript{0.01} & 0.01\textsubscript{0.01} \\
        \cmidrule(lr){2-8}
        & \multirow{5}{*}{\rotatebox[origin=c]{90}{\textbf{APS}}} & Avila (12) & 1.22\textsubscript{0.69} & 1.22\textsubscript{0.02} & 1.22\textsubscript{0.03} & 0.02\textsubscript{0.01} & 0.03\textsubscript{0.02} \\
        & & CardDefault (2) & 1.36\textsubscript{0.50} & 1.36\textsubscript{0.01} & 1.36\textsubscript{0.02} & 0.01\textsubscript{0.01} & 0.02\textsubscript{0.01} \\
        & & Landsat (6) & 1.32\textsubscript{0.78} & 1.32\textsubscript{0.03} & 1.31\textsubscript{0.05} & 0.03\textsubscript{0.02} & 0.04\textsubscript{0.03} \\
        & & LetterRecognition (26) & 2.49\textsubscript{2.63} & 2.49\textsubscript{0.07} & 2.50\textsubscript{0.10} & 0.05\textsubscript{0.04} & 0.08\textsubscript{0.06} \\
        & & MagicGamma (2) & 1.21\textsubscript{0.49} & 1.21\textsubscript{0.01} & 1.21\textsubscript{0.02} & 0.01\textsubscript{0.01} & 0.02\textsubscript{0.01} \\
        \bottomrule
    \end{tabular}
    \end{scriptsize}
    \end{center}
    
    \caption{
        \textbf{Marginal expected prediction set sizes (interval estimates).}
        We compare interval estimates for the marginal expected prediction set size (with $\gamma = 0.1$).
        They are obtained via our interval estimates, the central limit theorem (CLT), Hoeffding's inequality (HI), and Bernstein's inequality (BI).
        We compute their sizes and error frequencies in bounding the mean Monte Carlo average.
        We report the means and standard deviations.
    }
    \label{table:uci-other-interval-estimates}
    \vspace{-10pt}
    \begin{center}
    \begin{scriptsize}
    \begin{tabular}{lllC{1.5cm}C{1cm}C{1cm}C{1cm}C{1cm}C{1cm}C{1cm}C{1cm}}
        \toprule
        & & & \multicolumn{4}{c}{\textbf{Interval size}} & \multicolumn{4}{c}{\textbf{Interval error frequency}} \\
        \cmidrule(lr){4-7} \cmidrule(lr){8-11}
        & & \textbf{Dataset} & \textbf{Ours} & \textbf{CLT} & \textbf{HI} & \textbf{BI} & \textbf{Ours} & \textbf{CLT} & \textbf{HI} & \textbf{BI} \\
        \midrule
        \multirow{10}{*}{\rotatebox[origin=c]{90}{\textbf{Regression ($\pmb{\mathcal{Y} = \mathds{R}}$)}}} & \multirow{5}{*}{\rotatebox[origin=c]{90}{$\pmb{l_{1}}$}} & Abalone & 0.84\textsubscript{0.10} & 0.00\textsubscript{0.00} & & & 0.00 & 1.00 & & \\
        & & AirFoil & 0.91\textsubscript{0.11} & 0.00\textsubscript{0.00} & & & 0.00 & 1.00 & & \\
        & & AirQuality & 0.01\textsubscript{0.00} & 0.00\textsubscript{0.00} & & & 0.00 & 1.00 & & \\
        & & PowerPlant & 0.13\textsubscript{0.01} & 0.00\textsubscript{0.00} & & & 0.00 & 1.00 & & \\
        & & WhiteWineQuality & 0.70\textsubscript{0.07} & 0.00\textsubscript{0.00} & & & 0.00 & 1.00 & & \\
        \cmidrule(lr){2-11}
        & \multirow{5}{*}{\rotatebox[origin=c]{90}{\textbf{CQR}}} & Abalone & 0.64\textsubscript{0.03} & 0.13\textsubscript{0.01} & & & 0.00 & 1.00 & & \\
        & & AirFoil & 0.73\textsubscript{0.12} & 0.15\textsubscript{0.01} & & & 0.00 & 0.48 & & \\
        & & AirQuality & 0.00\textsubscript{0.00} & 0.01\textsubscript{0.00} & & & 1.00 & 0.12 & & \\
        & & PowerPlant & 0.10\textsubscript{0.01} & 0.02\textsubscript{0.00} & & & 0.00 & 0.52 & & \\
        & & WhiteWineQuality & 0.03\textsubscript{0.08} & 0.12\textsubscript{0.01} & & & 0.89 & 0.52 & & \\
        \midrule
        \multirow{15}{*}{\rotatebox[origin=c]{90}{\textbf{Classification (discrete $\pmb{\mathcal{Y}}$)}}} & \multirow{5}{*}{\rotatebox[origin=c]{90}{\textbf{0-1}}} & Avila (12) & 0.00\textsubscript{0.00} & 0.00\textsubscript{0.00} & 0.58\textsubscript{0.00} & 0.02\textsubscript{0.00} & 0.00 & 0.00 & 0.00 & 0.00 \\
        & & CardDefault (2) & 0.00\textsubscript{0.00} & 0.00\textsubscript{0.00} & 0.08\textsubscript{0.00} & 0.00\textsubscript{0.00} & 0.00 & 0.00 & 0.00 & 0.00 \\
        & & Landsat (6) & 4.95\textsubscript{0.21} & 0.00\textsubscript{0.00} & 0.52\textsubscript{0.00} & 0.03\textsubscript{0.00} & 0.00 & 1.00 & 1.00 & 1.00 \\
        & & LetterRecognition (26) & 5.50\textsubscript{5.85} & 0.00\textsubscript{0.00} & 1.27\textsubscript{0.00} & 0.04\textsubscript{0.00} & 0.00 & 0.00 & 0.00 & 0.00 \\
        & & MagicGamma (2) & 0.03\textsubscript{0.06} & 0.00\textsubscript{0.00} & 0.10\textsubscript{0.00} & 0.00\textsubscript{0.00} & 0.00 & 0.00 & 0.00 & 0.00 \\
        \cmidrule(lr){2-11}
        & \multirow{5}{*}{\rotatebox[origin=c]{90}{\textbf{LAC}}} & Avila (12) & 0.05\textsubscript{0.00} & 0.02\textsubscript{0.00} & 0.58\textsubscript{0.00} & 0.04\textsubscript{0.00} & 0.00 & 0.44 & 0.00 & 0.10 \\
        & & CardDefault (2) & 0.12\textsubscript{0.01} & 0.02\textsubscript{0.00} & 0.08\textsubscript{0.00} & 0.04\textsubscript{0.00} & 0.00 & 0.62 & 0.07 & 0.42 \\
        & & Landsat (6) & 0.15\textsubscript{0.01} & 0.03\textsubscript{0.00} & 0.52\textsubscript{0.00} & 0.06\textsubscript{0.00} & 0.00 & 0.58 & 0.00 & 0.26 \\
        & & LetterRecognition (26) & 0.08\textsubscript{0.01} & 0.02\textsubscript{0.00} & 1.27\textsubscript{0.00} & 0.06\textsubscript{0.00} & 0.00 & 0.50 & 0.00 & 0.07 \\
        & & MagicGamma (2) & 0.09\textsubscript{0.01} & 0.02\textsubscript{0.00} & 0.10\textsubscript{0.00} & 0.03\textsubscript{0.00} & 0.00 & 0.56 & 0.00 & 0.39 \\
        \cmidrule(lr){2-11}
        & \multirow{5}{*}{\rotatebox[origin=c]{90}{\textbf{APS}}} & Avila (12) & 0.14\textsubscript{0.01} & 0.05\textsubscript{0.00} & 0.58\textsubscript{0.00} & 0.08\textsubscript{0.00} & 0.01 & 0.48 & 0.00 & 0.22 \\
        & & CardDefault (2) & 0.12\textsubscript{0.01} & 0.03\textsubscript{0.00} & 0.08\textsubscript{0.00} & 0.04\textsubscript{0.00} & 0.00 & 0.54 & 0.08 & 0.36 \\
        & & Landsat (6) & 0.24\textsubscript{0.03} & 0.09\textsubscript{0.00} & 0.52\textsubscript{0.00} & 0.15\textsubscript{0.01} & 0.00 & 0.36 & 0.00 & 0.11 \\
        & & LetterRecognition (26) & 0.51\textsubscript{0.04} & 0.17\textsubscript{0.01} & 1.27\textsubscript{0.00} & 0.28\textsubscript{0.01} & 0.00 & 0.40 & 0.00 & 0.18 \\
        & & MagicGamma (2) & 0.11\textsubscript{0.01} & 0.03\textsubscript{0.00} & 0.10\textsubscript{0.00} & 0.05\textsubscript{0.00} & 0.00 & 0.41 & 0.01 & 0.20 \\
        \bottomrule
    \end{tabular}
    \end{scriptsize}
    \end{center}
    \vspace{-10pt}
\end{table}

The Monte Carlo average is the commonly used empirical estimate for the expected prediction set size.
In~\cref{section:experiments-marginal}, we compared our estimates from~\cref{section:size-estimation} with the Monte Carlo average, where the former used $3 \times$ fewer data points in its approximation.
Here we compare the two when the same data points are used for both.

We first describe how the Monte Carlo average is obtained.
This equates to sampling a (pseudo) calibration data and obtaining conformal prediction sets on multiple (pseudo) test data; the average size of the obtained prediction sets is the Monte Carlo average.
As we did before, we assume access to $k$ data points $Z^{\prime}_{1} = ( X^{\prime}_{1} , Y^{\prime}_{1} ), \ldots, Z^{\prime}_{k} = ( X^{\prime}_{k} , Y^{\prime}_{k} )$ that are available for the purpose of deriving estimates.
When $k > n$, we can sample $n$ data points to be the (pseudo) calibration data and the remaining $k - n$ data points to be the (pseudo) test data; this matches the number of calibration data used in the estimation and the one we want to estimate for.
However, when $k \leq n$ (which is often the case), we cannot avoid the mismatch in the calibration data used in the estimation and the one we want to estimate for.
Instead, we split the data into $k / 2$ each for both calibration and test.
As a result, the sizes of the obtained prediction sets are i.i.d.~(which will be useful later) and we call their average the same-data Monte Carlo average (since we will use the same $k$ accessible data points as the ones used for our estimates).

We follow the same setup as in~\cref{section:experiments-setup} to compare our point estimate and the same-data Monte Carlo average with respect to the regular Monte Carlo average;~\cref{table:uci-other-point-estimates} compares these estimates.
We observe that the means of both our point estimates and the same-data Monte Carlo averages are close to that of the Monte Carlo average, but our point estimates have comparable or smaller standard deviations.
When comparing the absolute errors between the individual estimates and the mean Monte Carlo average, the errors of our point estimates never exceed those of the same-data Monte Carlo averages.
This corroborates the practical use of our point estimates.

Furthermore, since the prediction set sizes are i.i.d.,~we can obtain confidence intervals for the expected prediction set size using concentration inequalities; we make use the following ones:~(i) the central limit theorem (CLT), (ii) Hoeffding's inequality (HI), and (iii) Bernstein's inequality (BI).
CLT is valid only asymptotically.
HI is finite-sample valid for random variables with known bounds, so is useful only for classification problems (where the set size is bound by $[ 0 , \lvert \mathcal{Y} \rvert ]$).
BI is finite-sample valid for random variables with known bounds and variances (useful only for classification problems); however, since the variance is unknown and needs to be approximated, the intervals may not be valid confidence intervals.
Alternatively, our interval estimates combine the expected set size in~\cref{theorem:size-quantification} and the Dvoretzky–Kiefer–Wolfowitz inequality~\citep{dvoretzky1956asymptotic,massart1990tight}.

We again follow the same setup as in~\cref{section:experiments-setup} to compare our interval estimates and the ones obtained using CLT, HI, and BI;~\cref{table:uci-other-interval-estimates} compares these estimates (with $\gamma$ set to 0.1).
We compute the interval sizes and the frequencies of error in bounding the mean Monte Carlo average.
We observe that the CLT and the BI intervals are small, but their error frequencies can be high.
On the other hand, our intervals and the HI intervals consistently achieve errors below the requirement of $\gamma = 0.1$.
Further, the HI intervals cannot be applied to regression problems, and they do not change with the non-conformity function as the intervals are determined by the classification problem (the number of labels/classes), the number of data points used in the approximation, and $\gamma$.
On the other hand, our intervals can be applied to regression problems, and they adapt with the non-conformity function being used, often resulting in smaller interval sizes.
Hence our estimated intervals provide practical use.

\section{EXPERIMENTS ON SYNTHETIC EXAMPLE}
\label{appendix:experiments-synthetic}

\begin{figure}[t]
    \begin{center}
        \includegraphics[width=\textwidth]{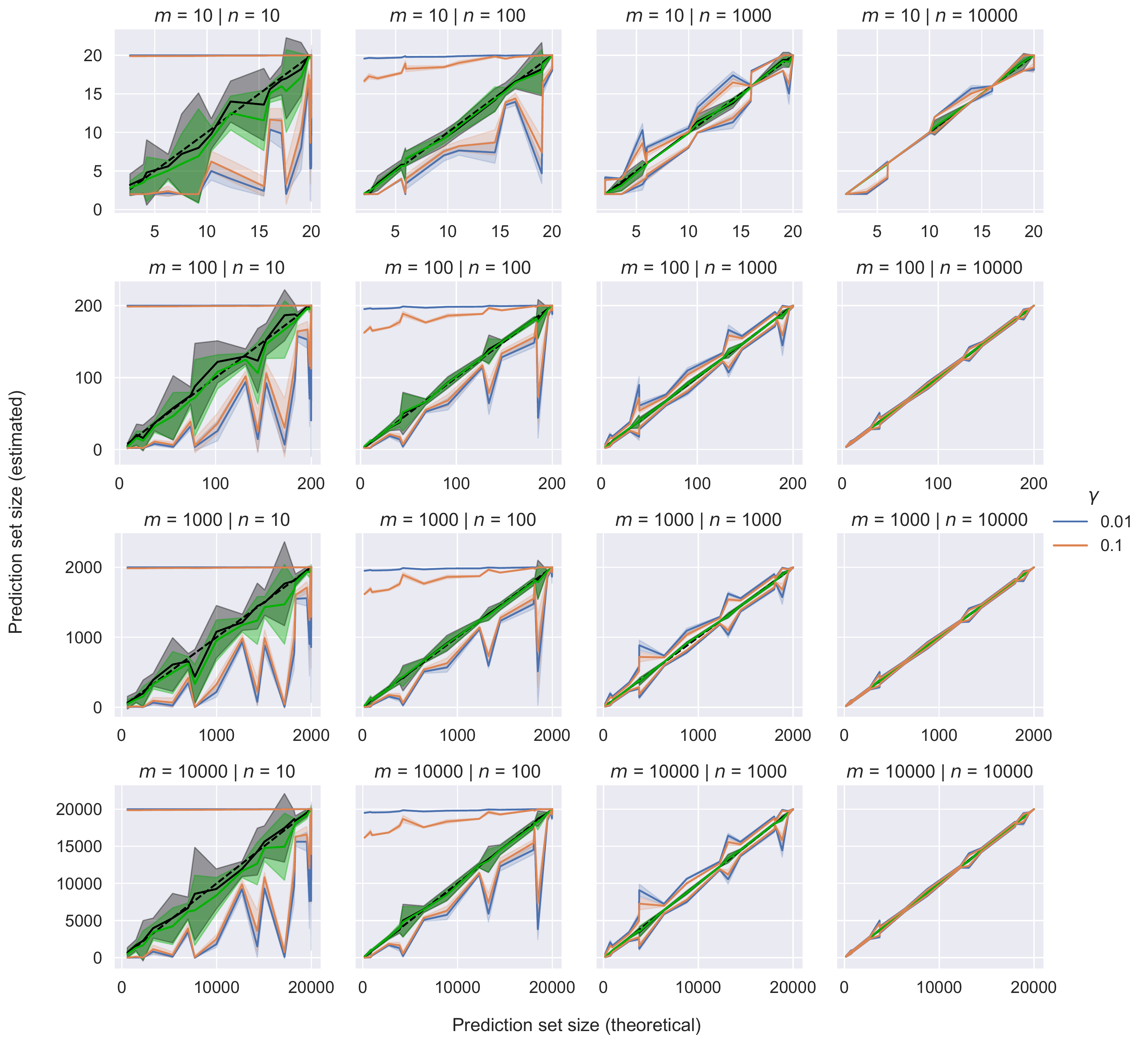}
    \end{center}
    \vspace{-10pt}
    \caption{
        \textbf{Marginal expected prediction set sizes (synthetic example).}
        We plot the theoretically expected prediction set sizes (cf.~\cref{equation:size-quantification-iid}) on the x-axis vs. its empirical estimates on the y-axis.
        These estimates include:~(i) the Monte Carlo average (solid black line), (ii) our point estimate from~\cref{subsection:size-estimation-known-multiplicative-factor} (green line), and (iii) our upper-lower confidence bounds from~\cref{corollary:size-confidence-interval} (orange/blue lines, changing with $\gamma$ as per the legend).
        $\alpha$ is set to 0.1 and the results are averaged to a line plot over different $a$ and $b$ values (bands denote the standard deviations).
        Additionally, the dashed black line is the identity line.
        The number of calibration data points $n$ increases from left to right.
        The size of the space of non-conformity scores $m$ increases from top to bottom.
    }
    \label{figure:experiments-synthetic}
    \vspace{-10pt}
\end{figure}

We experimentally validate our theoretical results for the marginal expected prediction set size in~\cref{theorem:size-quantification,corollary:size-confidence-interval} through a synthetic example.
We design a setup where the distribution of the calibration non-conformity scores is known.
We set the space of non-conformity scores to be a discrete set of $m$ values $\mathcal{R} = \{ r_{1} , \ldots , r_{m} \}$, where $r_{1} < \ldots < r_{m}$.
We set the distribution of the calibration non-conformity scores to be a beta-binomial distribution $\text{BetaBin} ( m - 1 , a , b )$ over the indices, with parameters $a , b > 0$.
Then, for a non-conformity score $r_{i} \in \mathcal{R}$, $\tilde{P}_{R} ( r_{i} ) = \mathds{P} \{ \text{BetaBin}(m - 1, a, b) \leq i - 2 \}$.
Additionally, we set $\mf = 2$.

For a fixed set of values of the parameters $a, b, m, n$, and $\alpha = 0.1$, we theoretically compute the expected prediction set size using~\cref{equation:size-quantification-iid}.
We compare this against:~(i) the average size of prediction sets constructed by running the conformal prediction algorithm (the Monte Carlo average), and (ii) our point estimates and confidence intervals (for a given $\gamma$) of the expected prediction set size computed using~\cref{subsection:size-estimation-known-multiplicative-factor}.
We randomly sample $n$ i.i.d.~non-conformity scores; we use these as the calibration non-conformity scores for the Monte Carlo average, and use the same as accessible non-conformity scores for our estimates.
We repeat the process 10 times using different random seeds.
Additionally, we vary the parameter values in the following way:~$a, b \in \{ .0625, .25, 1, 4, 16 \}$, $m, n \in \{ 10, 100, 1000, 10000 \}$, and $\gamma \in \{ 0.1, 0.01 \}$.
This results in a total of 800 settings, repeated 10 times each.

\cref{figure:experiments-synthetic} plots the theoretically expected prediction set sizes and their empirical estimates across these different settings, which we average over different $a$ and $b$ to obtain line plots.
The identity line (dashed black line) is the theoretically expected prediction set size from~\cref{equation:size-quantification-iid}.
We make the following observations.
(i) The Monte Carlo average (solid black line) collapses to the identity line as $n$ increases from left to right; this validates our quantification of the expected set size in~\cref{theorem:size-quantification}.
(ii) The average of our point estimates (green line) also collapses to the identity line as $n$ increases.
(iii) Our confidence intervals contain the identity line with high probability; for $\gamma = 0.1$ (orange lines) and $\gamma = 0.01$ (blue lines), the confidence intervals contain the theoretically expected size values $99.9 \%$ and $100.0\%$ of the time respectively.
Additionally, as $n$ increases, these confidence intervals collapse to the identity line.
These validate our estimates in~\cref{subsection:size-estimation-known-multiplicative-factor} and our result in~\cref{corollary:size-confidence-interval}.

}

\end{document}